\documentclass[conference]{IEEEtran}
\ifCLASSINFOpdf
   \usepackage[pdftex]{graphicx}
   \graphicspath{{./},{./plots/}}
\else
\fi
\usepackage[cmex10]{amsmath}
\usepackage{amsthm}

\usepackage{caption}
\usepackage[font=footnotesize]{subfig}
\usepackage{mathptmx}
\newtheorem{proposition}{Proposition}
\newtheorem{corollary}{Corollary}
\usepackage{xcolor}
\usepackage{cite}
\begin{document}
%
\title{Product Reservoir Computing: Time-Series Computation with Multiplicative Neurons}

\author{\IEEEauthorblockN{Alireza Goudarzi}
\IEEEauthorblockA{Department of Computer Science\\
University of New Mexico\\
Albuquerque, NM 87131\\
Email: alirezag@cs.unm.edu}
\and
\IEEEauthorblockN{Alireza Shabani}
\IEEEauthorblockA{Google Inc.\\
 Venice, CA 90291\\
Email: shabani@google.com}
\and
\IEEEauthorblockN{Darko Stefanovic}
\IEEEauthorblockA{Department of Computer Science and \\ 
Center for Biomedical Engineering\\
University of New Mexico\\
Albuquerque, NM 87131\\
Email: darko@cs.unm.edu}}


%


\maketitle

\begin{abstract}
Echo state networks (ESN), a type of reservoir computing (RC) architecture, are  efficient and accurate artificial neural systems for time series processing and learning. An ESN consists of a core of recurrent neural networks, called a reservoir, with a small number of tunable parameters to generate a high-dimensional representation of an input, and a readout layer which is easily trained using regression to produce a desired output from the reservoir states. Certain computational tasks involve real-time calculation of high-order time correlations, which requires nonlinear transformation either in the reservoir or the readout layer. Traditional ESN employs a reservoir with sigmoid or {\it{tanh}} function neurons. In contrast, some types of  biological neurons obey response curves that can be described as a product unit rather than a sum and threshold. Inspired by this class of neurons, we introduce a RC architecture with a reservoir of product nodes for time series computation. We find that the product RC shows many properties of standard ESN such as short-term memory and nonlinear capacity. On standard benchmarks for chaotic prediction tasks, the product RC  maintains the performance of a standard nonlinear ESN while being more amenable to mathematical analysis. Our study provides evidence that such networks are powerful in highly nonlinear tasks owing to high-order statistics generated by the recurrent product node reservoir.
\end{abstract}
\IEEEpeerreviewmaketitle

\section{Introduction}
Understanding contextual information processing in the brain is one of the goals of neuroscience \cite{Wang2001455}.
%
%
%
Dominey et al. \cite{Dominey:1995qo} proposed a simple model to explain the interaction between the prefrontal cortex, corticostriatal projections, and basal ganglia in context-dependent motor control of eyes. In this model, visual input drives stable activity in the prefrontal cortex, which is projected onto basal ganglia using learned interactions in the striatum. This model has also been used to explain higher-level cognitive tasks such as grammar comprehension in the brain \cite{10.1371/journal.pone.0052946}. 

More abstract versions of this model, Liquid State Machines \cite{Maass:2002p1444} and Echo State Networks \cite{Maass:2002p1444,Jaeger02042004}, were later introduced in the neural network community  and were subsequently unified under the name {\em reservoir computing} (RC) \cite{verstraeten2007}. In RC, a fixed high-dimensional recurrent network, called the reservoir, is driven by an input signal. An adaptive readout layer then combines the reservoir states to produce a desired output. Figure~\ref{fig:rcdyn} provides a conceptual illustration of RC. ESN implements this idea with a discrete-time recurrent network with linear or  activation functions and a linear readout layer being trained using regression. Many variations of ESN exist and have been successfully applied to  engineering tasks such as time series prediction and system identification \cite{springerlink:10.1007}.

Owing to fixed recurrent connections in ESN, its training is much more efficient than ordinary recurrent neural networks (RNN), making it feasible to use their power in practical applications. ESN's power in time series processing has been attributed to the reservoir's memory \cite{PhysRevLett.92.148102,Dambre:2012fk} and high-dimensional projection of the input, which acts like a temporal discriminant kernel \cite{Hermans:2011fk} that is present in the critical dynamical regime, where input perturbations in the reservoir dynamics neither spread nor die out \cite{Bertschinger:2004p1450,PhysRevE.87.042808,4905041020100501}.

A major research direction in RC is to study how the choice of reservoir and readout layer architecture may improve the performance in different tasks \cite{springerlink:10.1007}. Recent insights into the nature of computation in ESN \cite{PhysRevLett.92.148102,Dambre:2012fk,Goudarzi2014176} show that the readout layer learns the temporal correlations between the reservoir dynamics and the desired output. Traditional {\it{tanh}} activation function in the reservoir creates nonlinear correlations that are challenging to characterize mathematically  and may lead to unpredictable results \cite{DBLP:journals/nn/YildizJK12}.

Here, we propose that additive neurons with a thresholding transfer function can be replaced by multiplicative neurons and no additional nonlinearity. The use of product nodes in neural networks was introduced in \cite{Durbin:1989wu} in an effort to learn the suitable high-order statistics for a given task. It has been reported that most synaptic interactions are multiplicative \cite{McAdams01011999}. Examples of such  multiplicative scaling in visual cortex include gaze-dependent input modulation in parietal neurons \cite{Andersen25101985}, modulation of neuronal response by attention in the V4 area \cite{McAdams01011999} and the MT area \cite{Treue01091999}. In addition, locust visual collision avoidance mediated by LGMD neurons \cite{Gabbiani:2002kl}, optomotor control in flies \cite{Geiger:1974fr,Gtz:1975qf}, and barn owl's auditory localization in inferior colliculus (ICx) neurons can only be explained with multiplicative interactions \cite{Pena13042001}.

Another popular architecture which uses product nodes is the ridge polynomial network \cite{377967}. In this architecture the learning algorithm iteratively adds groups of product nodes with integer weights to the network to compute polynomial functions of the inputs. This process continues until a desired error level is reached. The advantage of the product node with variable exponent over the ones used in polynomial networks is that instead of providing fixed integer power of inputs, the network can learn the individual exponents that can produce the required pattern \cite{Giles:87}.

\begin{figure*}[!t]
\centering
\includegraphics[width=5in]{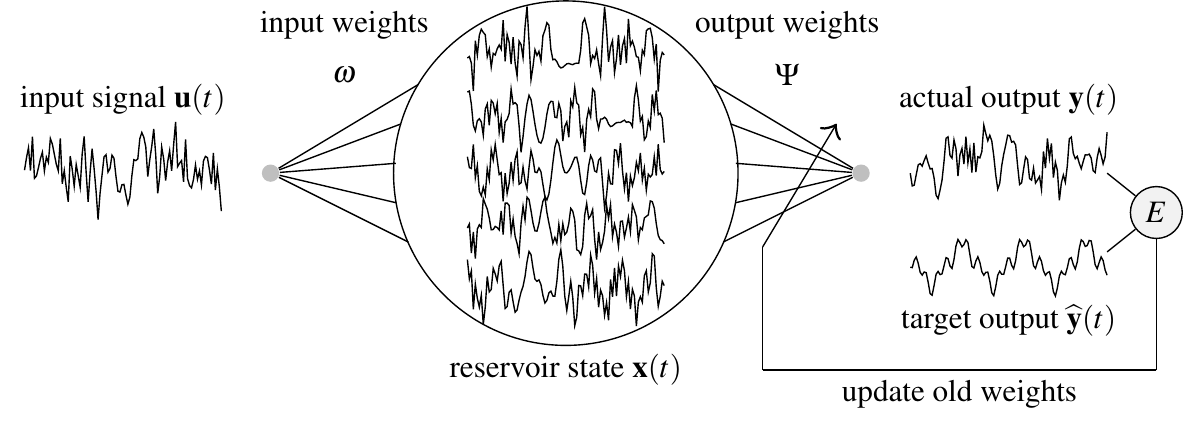}
\caption{Computation in an ESN. The reservoir is an excitable recurrent network with $N$ readable output states represented by the vector ${\bf X}(t)$. The input signal ${\bf u}(t)$ is fed into one or more points $i$ in the reservoir with a corresponding weight $\boldsymbol\omega_i$, denoted by the weight column vector ${\boldsymbol\omega}=[\boldsymbol\omega_i]$.}
\label{fig:rcdyn}
\end{figure*}

The main contribution of this work is to demonstrate the plausibility of product nodes for recurrent neural networks in the context of reservoir computing. In Section~\ref{sec:standard_esn} we review the basic ESN architecture that we use as a performance baseline to study product RC. In Section~\ref{sec:ff_fb_pu}, we describe the details of  product nodes and some practical considerations for their use in a recurrent neural network. In Section~\ref{sec:pu_esn}, we present our proposed architecture for reservoir computing using product nodes, specifically, the replacement of {\it{tanh-}} nodes in the ESN reservoir with product nodes and adjusting the initialization strategy accordingly. We show how to use the exponential-of-log trick to simulate the dynamics of product recurrent networks efficiently using ordinary matrix products. We also prove the echo state property for the network to show that the network dynamics is insensitive to initial conditions. The experimental study on information processing properties of product RCs is presented in Section~\ref{sec:results}. We first study the memory and also nonlinear memory capacity of product RCs, then we evaluate how such networks perform in predicting Mackey-Glass and Lorenz systems. Our results show that the product RC achieves performance similar to ESN with {\it{tanh}} functions.

\section{Model}
\label{sec:model}

\subsection{Echo State Network}
\label{sec:standard_esn}
An ESN consists of an input-driven recurrent neural network, which acts as the reservoir, and a readout layer that reads the reservoir states and produces the output. Mathematically, the input driven reservoir is defined as follows. Let $N$ be the size of the reservoir. We represent the time-dependent inputs as a column vector $u(t)$, the reservoir state as a column vector $x(t)$, and the output as a column vector $y(t)$. The input connectivity is represented by the matrix $\boldsymbol\omega$ and the reservoir connectivity is represented by an $N\times N$ weight matrix $\boldsymbol\Omega$. For simplicity, we assume that we have one input signal and one output, but the notation can be extended to multiple inputs and outputs. The time evolution of the reservoir is given by:
\begin{equation}
x(t+1) = f(\boldsymbol\Omega x(t) + \boldsymbol\omega u(t)).
\end{equation}
where $f$ is the transfer function of the reservoir nodes that is applied element-wise to its operand. An optional constant $b$ can be added to the operand to serve as the bias to the reservoir nodes. The transfer function $f$ is usually {\it{tanh}} or linear functions. The output is generated by the multiplication of  a readout weight matrix $\boldsymbol\Psi$ of length  $N+1$ and the reservoir state vector $x(t)$, extended by an optional constant $1$, represented by $x'(t)$:
\begin{equation}
y(t) = \boldsymbol\Psi   x'(t).
\end{equation}

The readout weights $\boldsymbol\Psi$ need to be trained using a teacher input-output pair. A popular training technique is to use the pseudo-inverse method\cite{verstraeten2007}. To use this method, one would drive the ESN with a teacher input and record the history of the reservoir states into a matrix ${\bf X}$, where the columns correspond to the reservoir nodes and the rows are the states of each reservoir node in time. 
The corresponding teacher output will be denoted by the column vector ${\bf \widehat{y}}$. The readout can be calculated as follows:
\begin{equation}
\Psi = \langle {\bf XX'}\rangle^{-1} \langle {\bf X} {\bf \widehat{Y}'}\rangle,
\label{eq:regression}
\end{equation}
where $'$ indicates the transpose of a matrix. Figure~\ref{fig:arch_esn} show the architecture of ESN. We will compare these two architectures with our proposed product node ESN architecture.

\subsection{Feed-forward and Feedback Product Nodes}
\label{sec:ff_fb_pu}
One of the goals of neural network research is to discover how high-order interactions can be represented using simple nodes inspired by neurons. Following the observation of multiplicative interactions in NDMA receptors at the level of a single neuron \cite{Collingridge:px}, Durbin and Rumelhart \cite{Durbin:1989wu} proposed product nodes in which different stimuli are raised to a power given by the respective synaptic weights and multiplied together. Single-neuron multiplicative interactions of this type have also been observed in owl ICx neurons \cite{Pena13042001} and locust LGMD neurons \cite{Gabbiani:2002kl}. Before giving a prescription for product RC, we briefly review the properties of a single product node with feed-forward and feedback connections. The original product node was defined as follows \cite{Durbin:1989wu}: 
\begin{align*}
x  = f(u_1^{\omega_1}u_2^{\omega_2}\dots u_N^{\omega_N})=f(\Pi_{i=1}^{N} u_i^{\omega_i}),
\end{align*}
where $x$ is the output of the product node, $f$ is the activation function, $u_i$ represent $N$ different input stimuli, and $\omega_i$ the corresponding weights on the inputs. We generalize this model to include a feedback connection by which the node can use multiplicative interaction with its input history to produce the output. This can be represented as follows:
\begin{align*}
x(t)  &= f(x(t-1)^\Omega u(t-1)_1^{\omega_1}u(t-1)_2^{\omega_2}\dots u(t-1)_N^{\omega_N}) \\
      &= f(x(t-1)^\Omega \Pi_{i=1}^N u(t-1)_i^{\omega_i}),
\end{align*}

\noindent where $\Omega$ represents the weight of the feedback connection. Note that the multiplicative coupling imposes some additional constraints on the admissible ranges for $\Omega$, $\omega$, $u$. Namely, a zero value in $x(t)$ and/or $u(t)$ at any point in time forces a reset of the entire history in the node, killing the value of its short-term memory. Moreover, the memory of an entire network of product nodes will be erased if a single node becomes zero. To achieve short-term memory with  multiplicative feedback nodes, we must choose the exponents such that the old inputs approach the value 1 and their effect diminishes over time. A possible choice is $u(t) \in (0,1]$, $\Omega \in (0,1]$, and $\omega \in (0,1]$. It is noteworthy that $u(t) \in [-1,0)$ is an admissible input, but will result in complex values that could be interpreted as a mechanism for simultaneously encoding firing rate and phase information in a biological neuron \cite{Reichert2014}. The complete analysis of the effect of such a behavior is beyond the scope of this work.  Figure~\ref{fig:pu_values} illustrates the output values of a product node $u^\omega$ for positive and negative input domains and different input weights $\omega$. 

\begin{figure}[h]
\centering
\subfloat[nonlinear  reservoir]{
\includegraphics[width=2.9in]{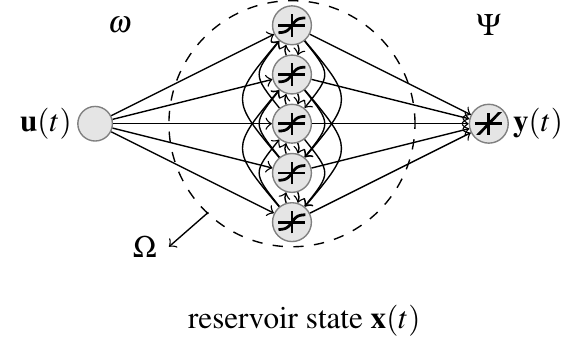}%
}\\
\caption{Schematic of an ESN. A time-varying input signal ${\bf u}(t)$ derives a dynamical core called a reservoir. The states of the reservoir ${\bf x}(t)$  are combined linearly to produce the output ${\bf y}(t)$. The reservoir consists of $N$ nodes. The input and the reservoir connections are given by the vector $\boldsymbol\omega$ and the matrix $\boldsymbol\Omega$ respectively. The reservoir states and the constant are connected to the readout layer using the weight matrix $\boldsymbol\Psi$.
}
\label{fig:arch_esn}
\end{figure}


Another practical consideration is that as the feedback weight $\Omega$ approaches $1$, the  output of the product unit will approach $0$ due to its long multiplicative history in the range $(-1,1)$. This is similar to the saturating effect of a {\it{tanh}} function in a standard ESN. We found that for $\Omega > 0.8$ the dynamics of a feedback node is not suitable for storing its input history.

\begin{figure}[t]
\centering
\subfloat[$u=0.1$]{
\includegraphics[width=1.75in]{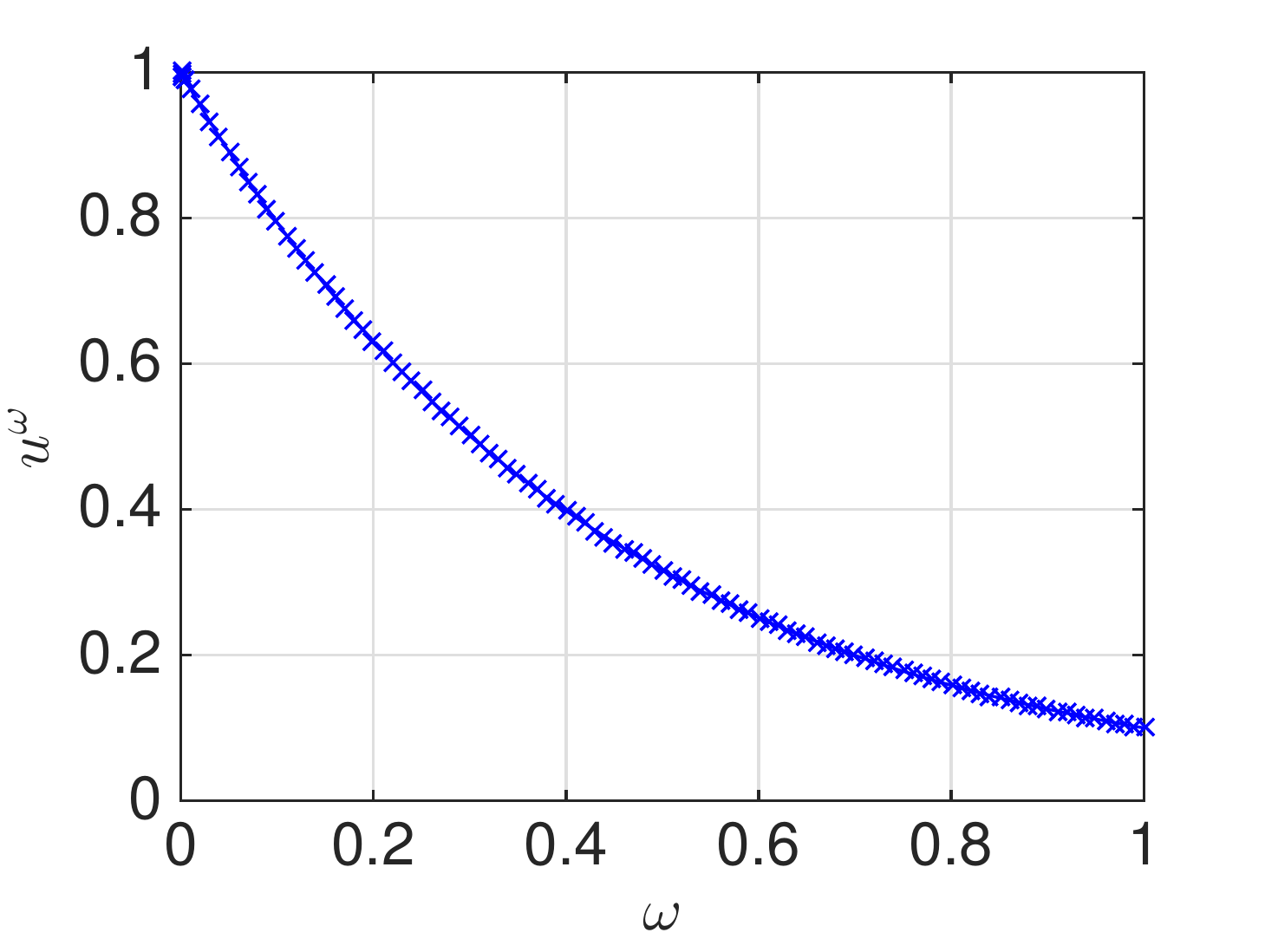}%
\label{fig:pu_values2}
}
\subfloat[$u=-0.1$]{
\includegraphics[width=1.75in]{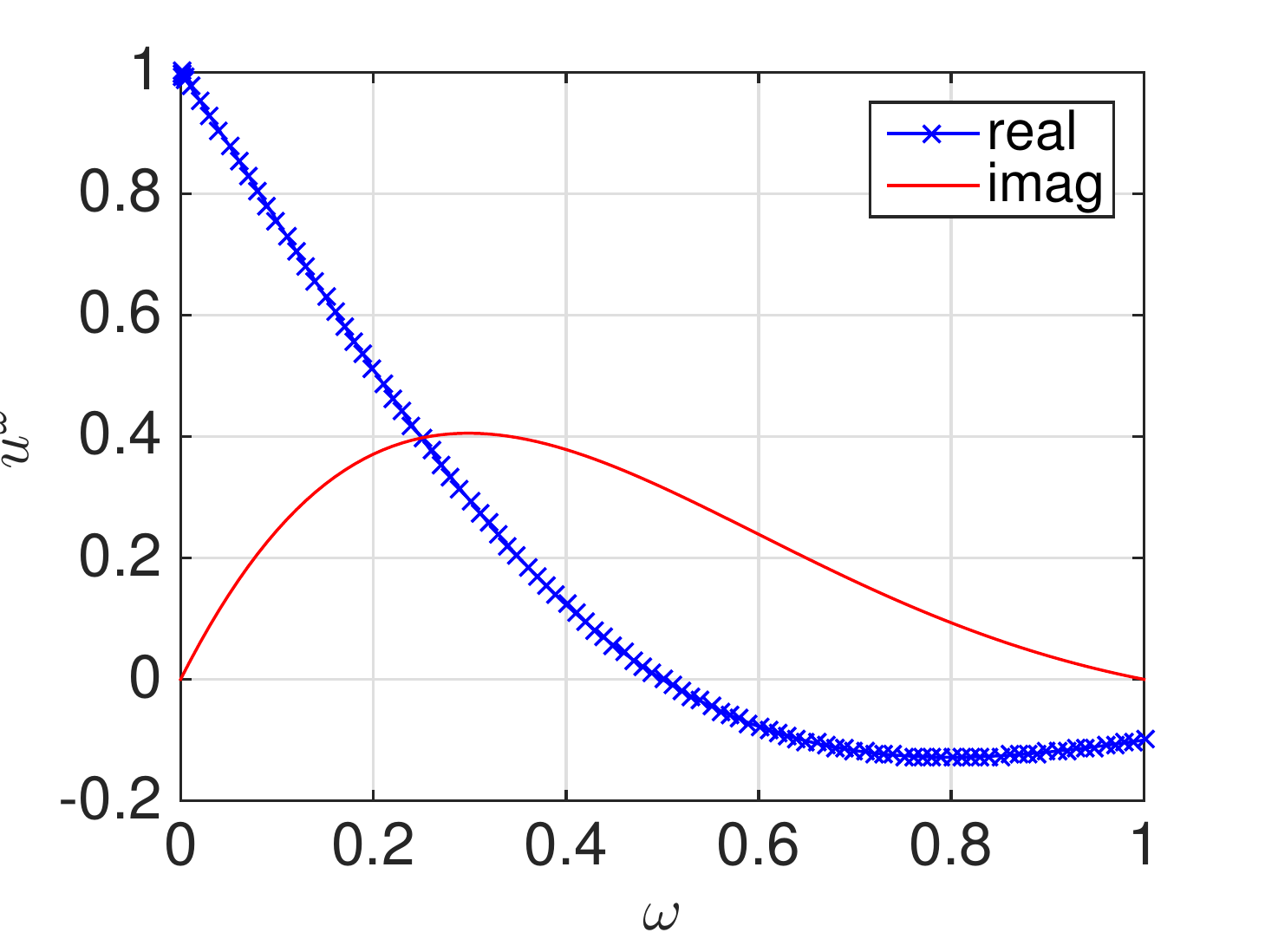}%
\label{fig:pu_values1}}
\caption{Example of the behavior of a product node with positive and negative input $u$ for different input weights $\omega\in [0,1]$. The complex values may be interpreted as simultaneously encoding for firing rate and phase information.}
\label{fig:pu_values}
\end{figure}

\subsection{ESN Architecture with Recurrent Product Network}
\label{sec:pu_esn}

We now consider the general case of a network with multiple nodes. For simplicity we use product nodes with linear activation function in the reservoir and a linear readout layer trained with ordinary regular regression. For very small input weights, {\it{tanh-}} ESN behave very similar to linear ESN; an appropriate combination of input weights scaling and reservoir bias $b$ can map the inputs onto the nonlinear regions of the {\it{tanh}} function, which dramatically improves the performance of the ESN for nonlinear tasks \cite{5596492}. We will later show that despite linear activation, this architecture achieves a similar performance to ESN with {\it{tanh}} activations on standard benchmark tasks.

We use $N$ coupled product nodes with linear activation to build a recurrent product network as our reservoir. The coupling is given by an $N\times N$ matrix $\boldsymbol\Omega=[\Omega_{i,j}]$, where $\Omega_{i,j}$ is the weight from node $j$ to $i$. Each node also receives a connection from the input $u(t)$. The input connectivity is given by the vector $\boldsymbol\omega = [\omega_i]$, where $\omega_i$ is the weight from input to the reservoir node $i$. Without loss of generality we restrict ourselves to networks with one input and one output. The state of the reservoir at each time  is given by the vector ${\bf x}(t)=[x_i]$, where $x_i$ is the node $i$. We assume both inputs and reservoir states are defined over compact sets. The time evolution of each node $i$ is given by: 
\begin{equation}
x_i(t) = \Pi_{j=0}^N x_j(t-1)^{\Omega_{i,j}} u(t-1)^{\omega_i}.
\label{eq:nodeevol}
\end{equation}

The following proposition gives us a way to simulate the network dynamics using normal matrix product.

\begin{proposition} 
Global dynamics of the recurrent product network given by Equation~\ref{eq:nodeevol} can be expressed as follows:
\begin{equation}
{\bf x}(t) = \exp\left(\boldsymbol\Omega\log {\bf x}(t-1) + \boldsymbol\omega\log u(t-1)\right),
\label{eq:prop1}
\end{equation}
where the $\log$ and $\exp$ are applied element-wise.
\label{prop1}
\end{proposition}
\begin{proof}
Recall that the dynamics of reservoir nodes is given by the following system: 
\begin{align*}
x_1(t) &= \Pi_{j=0}^N x_j(t-1)^{\Omega_{1,j}} u(t-1)^{\omega_1} \\
x_2(t) &= \Pi_{j=0}^N x_j(t-1)^{\Omega_{2,j}} u(t-1)^{\omega_2} \\
  \vdots \\
x_N(t) &= \Pi_{j=0}^N x_j(t-1)^{\Omega_{N,j}} u(t-1)^{\omega_N}.
\end{align*}
Taking the logarithm of both sides of each equation we have: 
\begin{align*}
\log x_1(t) &= \sum_{j=0}^N \Omega_{1,j} \log x_j(t-1) + \omega_1\log u(t-1) \\
\log x_2(t) &= \sum_{j=0}^N \Omega_{2,j} \log x_j(t-1) + \omega_2 \log u(t-1)\\
  \vdots \\
\log x_N(t) &= \sum_{j=0}^N \Omega_{N,j} \log x_j(t-1) + \omega_N \log u(t-1).
\end{align*}

This  can be rewritten in compact matrix form as:
\begin{align*}
\log {\bf x}(t) = \boldsymbol\Omega \log {\bf x}(t-1) + \boldsymbol\omega \log u(t-1).
\end{align*}
Finally, an element-wise exponentiation of both sides will give us:
\begin{align*}
{\bf x}(t) = e^{\boldsymbol\Omega\log {\bf x}(t-1) + \boldsymbol\omega\log u(t-1)} \qedhere 
\end{align*}
\end{proof}
\begin{figure*}
\centering
\def\w{2in}
\subfloat[product linear]{
\includegraphics[width=\w]{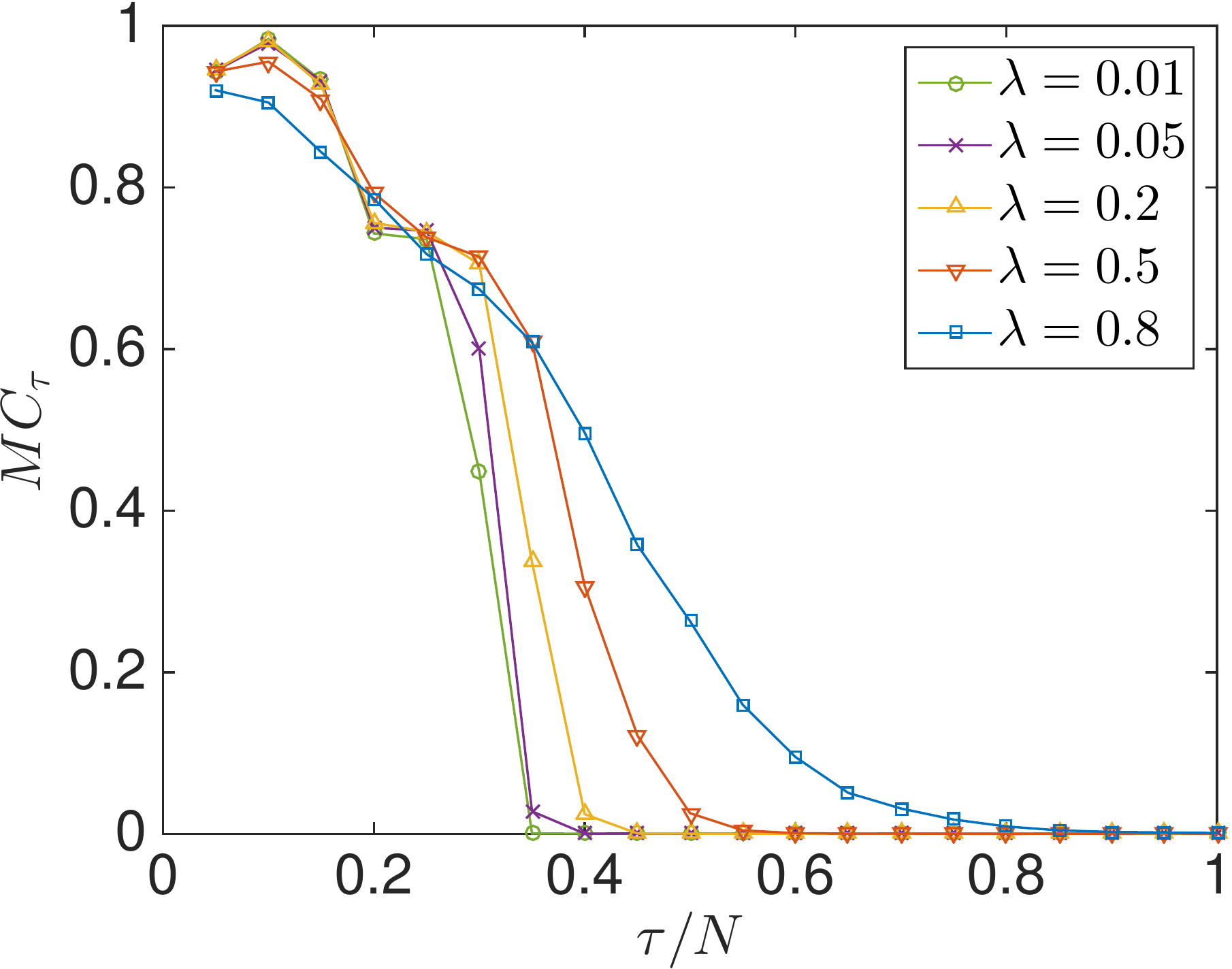}
}
\subfloat[linear]{
\includegraphics[width=\w]{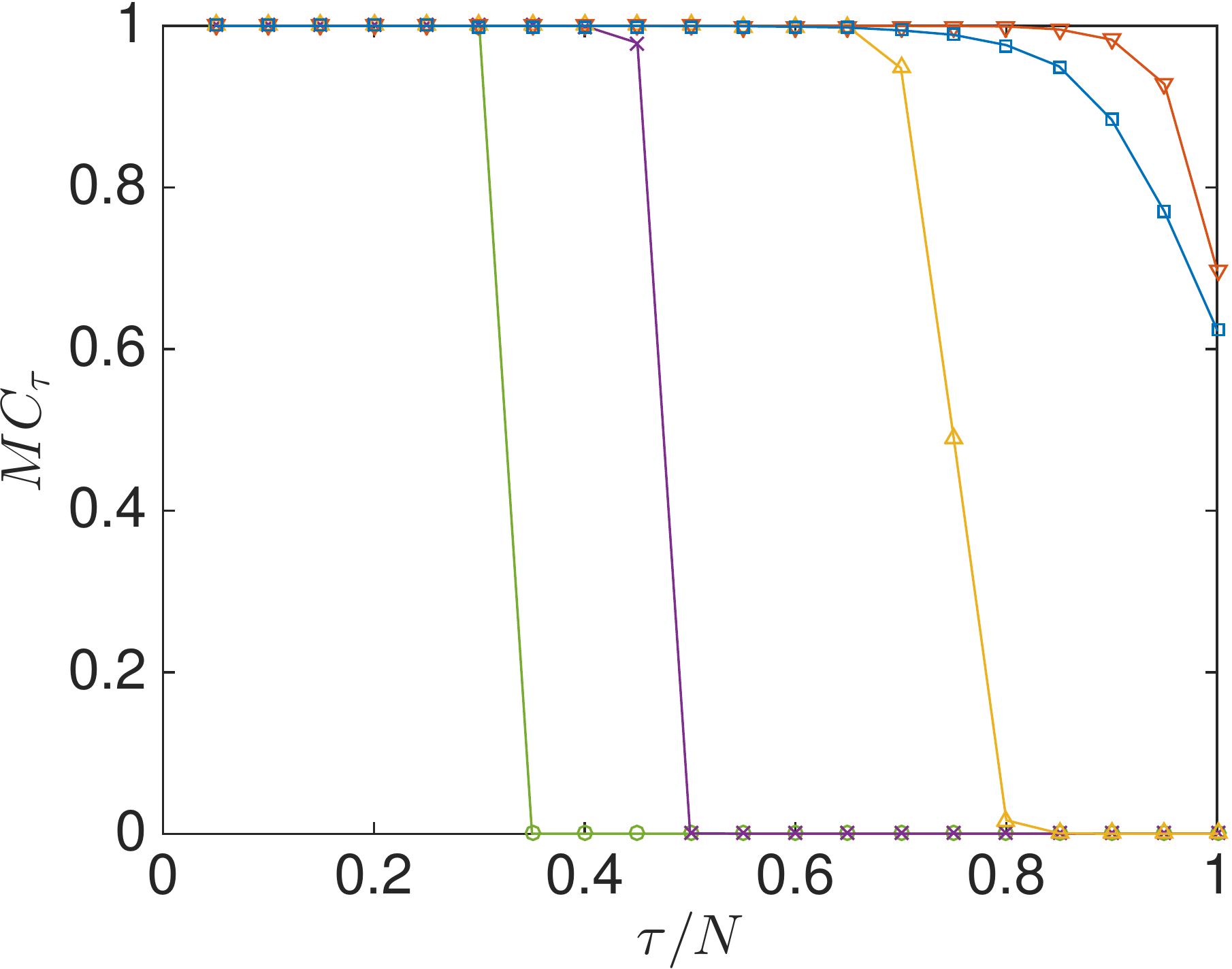}
}
\subfloat[tanh]{
\includegraphics[width=\w]{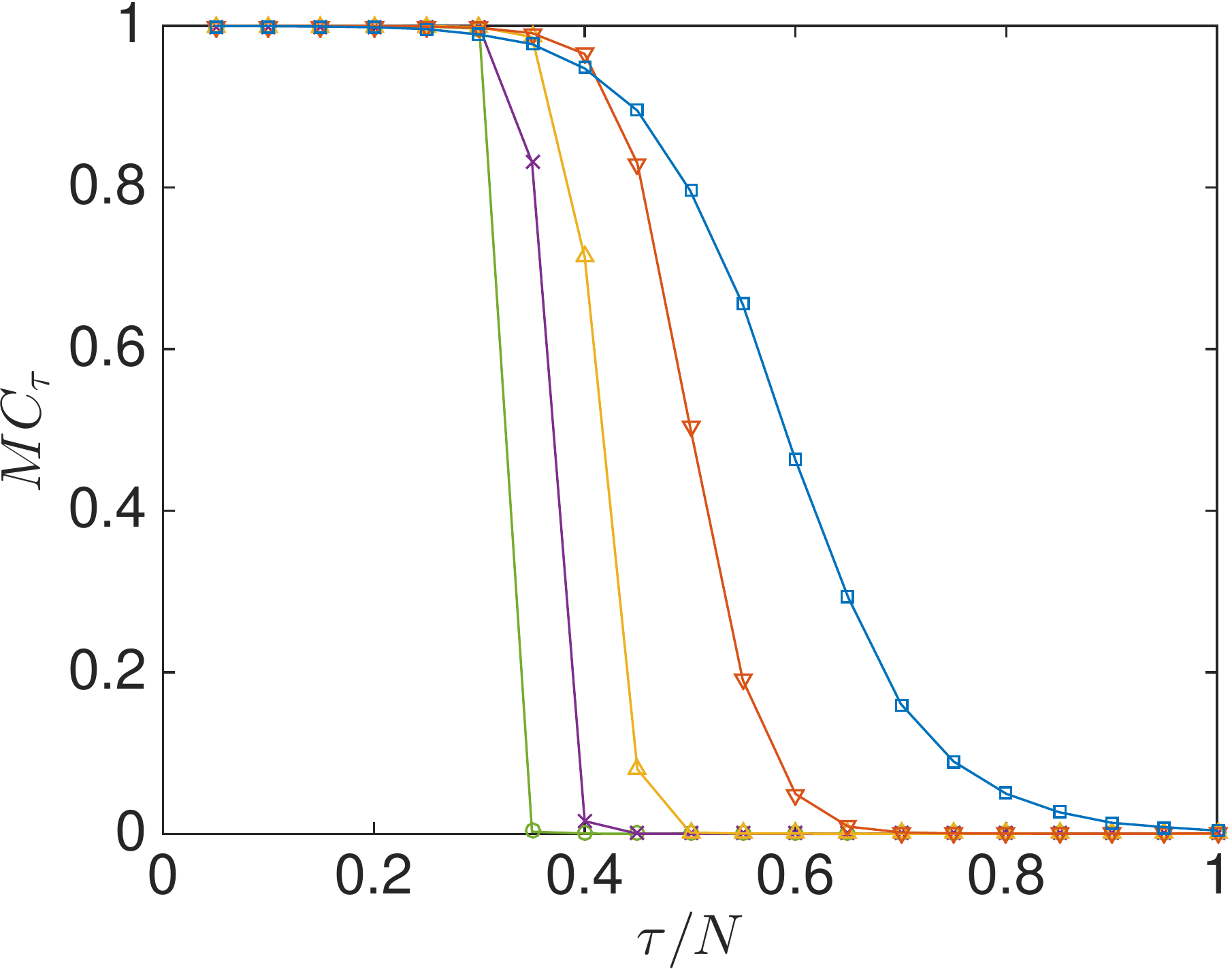}
}
\caption{The linear memory capacity of the product RC and the standard ESN with linear and {\it{tanh}} activation functions for fixed input coefficient $\omega=0.2$ and different $\lambda$. As expected the more nonlinearity in the system the faster the memory curves decrease, i.e., {\it{tanh-}} ESN curves decrease faster than the linear ESN and the product RC faster than the {\it{tanh-}} ESN. 
}
\label{fig:MCfigures}
\end{figure*}

\begin{corollary}
Given a recurrent product network with dynamical equation described by Equation~\ref{eq:nodeevol}, the state of the network at a given time can be explicitly written as a function of the initial state of the network and its input history as follows:
\begin{align}
{\bf x}(t) = \exp\left(\boldsymbol\Omega^t\log {\bf x}(0) + \sum_{i=0}^{t-1}\boldsymbol\Omega^{t-i-1}\boldsymbol\omega\log u(i)\right)
\end{align}
\label{cor1}
\end{corollary}
\begin{proof}
This can be easily verified by expanding the recursion in Equation~\ref{eq:prop1}:
\begin{align*}
{\bf x}(1) &= \exp\left(\boldsymbol\Omega\log {\bf x}(0) + \boldsymbol\omega\log u(0)\right),\\
{\bf x}(2) &= \exp\left(\boldsymbol\Omega\log {\bf x}(1) + \boldsymbol\omega\log u(1)\right),\\
           &= \exp\left(\boldsymbol\Omega^2\log {\bf x}(0) + \boldsymbol\Omega\boldsymbol\omega\log u(0) + \boldsymbol\omega\log u(1)\right),\\
 {\bf x}(t) &= \exp\left(\boldsymbol\Omega^t\log {\bf x}(0) + \sum_{i=0}^{t-1}\boldsymbol\Omega^{t-i-1}\boldsymbol\omega\log u(i)\right)
\end{align*}

\qedhere
\end{proof}
Computation in ESNs is enabled by an important property which ensures that the reservoir state is asymptotically only a function of its input history. This is called the echo-state property (ESP). In \cite{Jaeger:2001p1442} two conditions
were stipulated for a recurrent network given by a weight matrix $\boldsymbol\Omega$ to hold the ESP: (1) a necessary condition that the spectral radius of $\boldsymbol\Omega$ should not be greater than unity; and (2)  a sufficient condition that the largest singular value of $\boldsymbol\Omega$ should be less than unity. Later, the sufficient condition was deemed too conservative and was updated~\cite{1629106} and the necessary condition was shown to be statistically enough for a good reservoir~\cite{6105577}. Yildiz et al. \cite{DBLP:journals/nn/YildizJK12} presented a pathological example to demonstrate that for an ESN with $\tanh$ functions neither of the conditions guarantees the ESP. However, this only holds for nonlinear systems and for a linear system the weight matrix spectral radius less than unity is enough to guarantee the ESP. The following corollary builds on Corollary~\ref{cor1} and gives us an equivalent of the ESP for recurrent product networks.
\begin{proposition} 
Given a recurrent product network described by Equation~\ref{eq:nodeevol}, the assumed compactness conditions on the inputs and the network state, and a recurrent weight matrix $\boldsymbol\Omega$ with spectral radius $\lambda<1$, the asymptotic global dynamics of the network is only   a function of the input history $u(t)$.
\label{prop2}
\end{proposition}
\begin{proof}
First, we note that the dynamics of $\log x(t)$ is linear. In addition, the unity vector ${\bf 1}$ is the nullspace of the system $\boldsymbol\Omega\log {\bf x}(t_0)$ and that $\lim_{t\to\infty} \boldsymbol\Omega^t \log {\bf x}(t_0)=0$. Using Corollary~\ref{cor1} we can write the state of the system at time $t\to\infty$ as follows:
\begin{align*}
\lim_{t\to \infty} {\bf x}(t) &= \lim_{t\to\infty}  \exp\left(\boldsymbol\Omega^t\log {\bf x}(0) + 
 \sum_{i=0}^{t-1}\boldsymbol\Omega^{t-i-1}\boldsymbol\omega\log u(i)\right)\\
 &= \exp\left(\sum_{i=0}^{t-1}\boldsymbol\Omega^{t-i-1}\boldsymbol\omega\log u(i)\right),
\end{align*}
which is a function of only the input history.\qedhere
\end{proof}

We should point out that the derivation of ESP is usually presented in terms of asymptotic difference between the states of two identical ESNs driven by identical inputs that are initialized in different states, i.e., $\lim_{t\to\infty} || x_1(t) - x_2(t)|| = 0$, where $x_1(t)$ and $x_2(t)$ refer to the long-term state of the ESN  initialized with different random values.  It is easy to see that this definition is equivalent to Proposition~\ref{prop2}. We  emphasize that since the systems dynamics is linear in the logarithm of the reservoir states and the unity vector ${\bf 1}$ is the global attractor, Proposition~\ref{prop2} constitutes a necessary and sufficient condition for the ESP in product RCs with linear transfer function.

\section{Experiments}
\label{sec:results}
In this section, we will compare the standard ESN, with linear and {\it{tanh}} activation functions, with the product RC. We will compare the performance of networks on computational capacity tasks and chaotic time-series prediction benchmarks.  

\subsection{Reservoir Construction and Evaluation}
For our experiments, we use fully connected reservoirs with $N$ nodes. The number of reservoir nodes $N$ is adjusted for each task to get reasonably good results in a reasonable amount of time. The reservoir weights $\Omega$ and input weights $\omega$ are drawn from i.i.d. normal distribution with mean zero and standard deviation 1, i.e., $\mathcal{N}(0,1)$. The reservoir is then rescaled to have spectral radius $\lambda$, while the input weights are multiplied by a coefficient $\omega$. For the {\it{tanh}} and linear reservoirs, the reservoir nodes are initialized with $0$s, and for the product reservoirs they are initialized with $1$s.

The reservoirs are driven with task-dependent input $u_t$ for $2,000$ time steps and the readout weights $\Psi$ are calculated as described in Section~\ref{sec:standard_esn} using MATLAB's {\em pinv()} function. For evaluation, the reservoir state is reinitialized and the reservoir is driven for another $T=2,000$ time steps and the output $y_t$ is generated. For brevity, throughout the experiments section we adopt the subscript notation for the time index, e.g. $y_t$ instead of $y(t)$. By convention, the system performance for computational capacity tasks is evaluated using the capacity function $MC_\tau$, which is the coefficient of determination between the output $y_t$ and the desired output $\widehat{y}_t$:
\begin{equation}
MC_{\tau} = \frac{\mathrm{Cov}^2(y_t,\widehat{y}_t)}{\mathrm{Var}(y_t)\mathrm{Var}(\widehat{y}_t)},
\end{equation}
where $\widehat{y}_t=\widehat{y}_t(u_{t-\tau})$ is a function of delayed input $u_{t-\tau}$ and $\tau$ is the memory length for the task. Total capacities are calculated by summing the capacity function over $\tau$: $MC=\sum_\tau MC_\tau$. We use $1\le\tau\le50$ for our empirical estimations. 
Note that for negative inputs to the product RC result in complex-valued outputs, capacities, and errors. Therefore, one must use the modulus of $MC_\tau$ and $NMSE$ for the product RC.

For the chaotic prediction task, the performance is evaluated by calculating the normalized mean-squared-error $NMSE$ as follows:
\begin{equation}
NMSE = \frac{\sqrt{ \frac{1}{T} \sum_{t=0}^{T} (y_t - \widehat{y}_t)^2}}{\mathrm{Var}(\widehat{y}_t)},
\label{eq:nmse}
\end{equation}
where $y_t$ is the network output and $\widehat{y_t}$ is the desired output.

\subsection{Computational Capacity}
\label{sec:membench}

Computational capacity tasks consist of linear memory and nonlinear capacities, which measure how well an ESN can reconstruct a function of its previous inputs. In these sets of experiments reservoirs of size $N=20$ nodes are driven with a one-dimensional input drawn from  uniform distributions on  $(0,1]$.  We systematically choose the input weight coefficients and reservoir spectral radius in the ranges $0.001<\omega<1$ and $0.01<\lambda<0.95$. The intervals are chosen from preliminary experiments to capture the regions of the parameter space where we get the best results or variation in their trends for the purpose of sensitivity analysis. All the results are averaged over 50 runs. The desired output of memory capacity is defined below.

\subsubsection{Linear Memory Capacity}
The linear memory capacity is a standard measure of memory in recurrent neural networks. The $\tau$-delay memory function $MC_\tau$  measures how long the network can remember its inputs. The desired output for this task is defined as:
\begin{equation}
\widehat{y}_t = u_{t-\tau}.
\end{equation}

\subsubsection{Nonlinear Computation Capacity}
The nonlinear computation capacity measures the ability of the system to reconstruct a nonlinear function of its past inputs. Conventionally, Legendre polynomials are used to calculate the nonlinear computation capacity of the reservoir \cite{Dambre:2012fk}; their advantage is that Legendre polynomials of different orders are orthogonal to each other, allowing one to measure the reservoir's capacity to compute functions of varying degrees of nonlinearity independently from each other. The desired output of the Legendre polynomial of order $n$ with delay $\tau$  is given by: 
\begin{equation}
\widehat{y}(n,\tau)_t = \frac{1}{2^n} \sum_{k=0}^n {n \choose k}^2 (u_{t-\tau}-1)^{n-k}(u_{t-\tau}+1)^{k}.
\end{equation}

\subsubsection{Results}
Figure~\ref{fig:MCfigures} shows the results of linear memory capacity experiments for different architectures and reservoir spectral radii. The input coefficient is fixed at $\omega=0.2$.  The $x$-axis shows the time delay as a ratio $\tau/N$ and the curves are averaged over 50 experiments. In general the product RCs show faster decrease in the $MC_\tau$, due to the product nonlinearity. This is  similar in nature to the effect that the saturated {\it{tanh}} activation function has on memory capacity. We then calculate the empirical total memory $MC = \sum_{\tau=0}^{50} MC_\tau$ for different values of $\omega$ and $\lambda$. 

\begin{figure}[h]
\centering
\def\w{1.6in}
\subfloat[total nonlinear capacity]{
\includegraphics[width=\w]{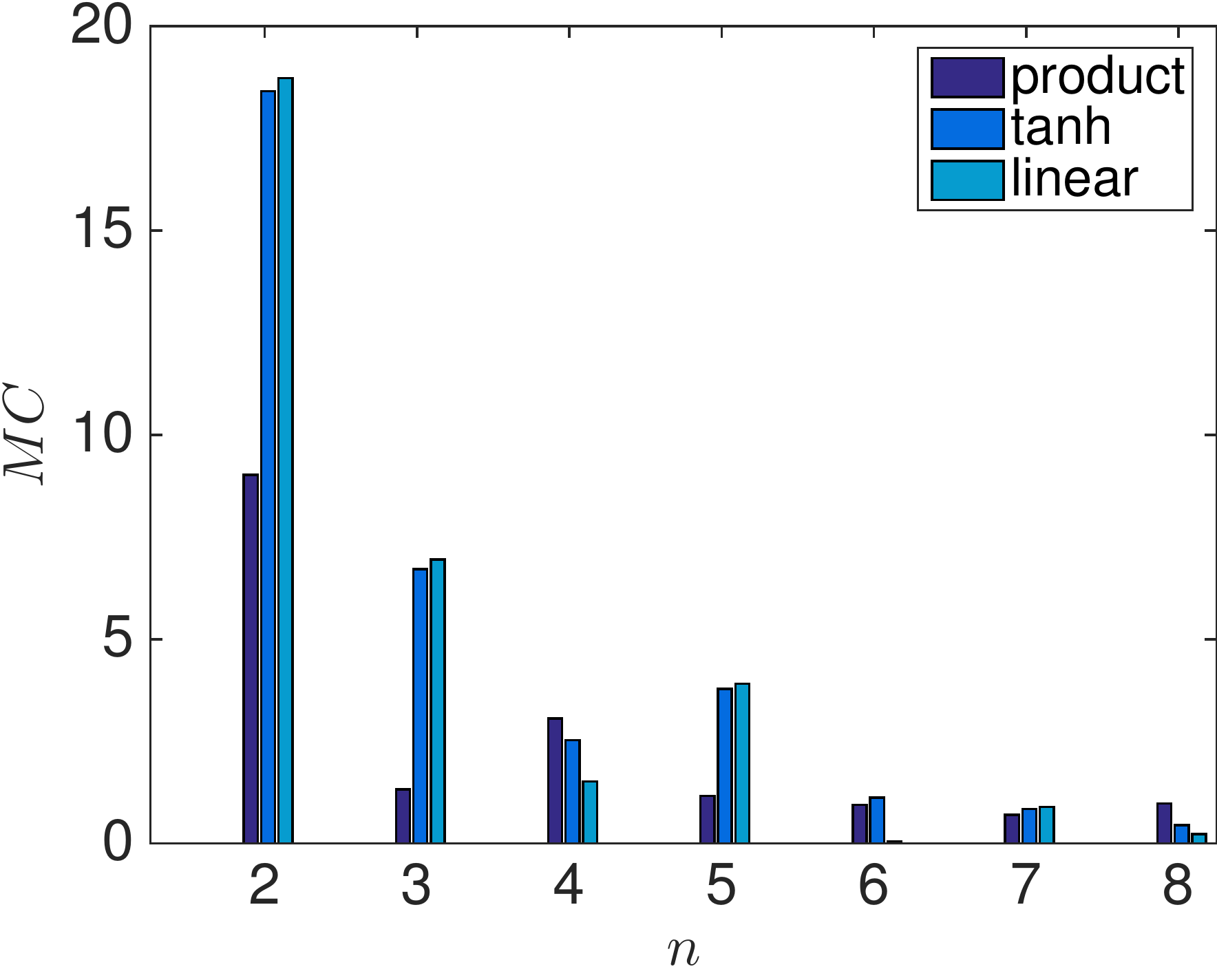}
\label{fig:bestNMC}
}
\subfloat[nonlinear capacity function for $n=3$]{
\includegraphics[width=\w]{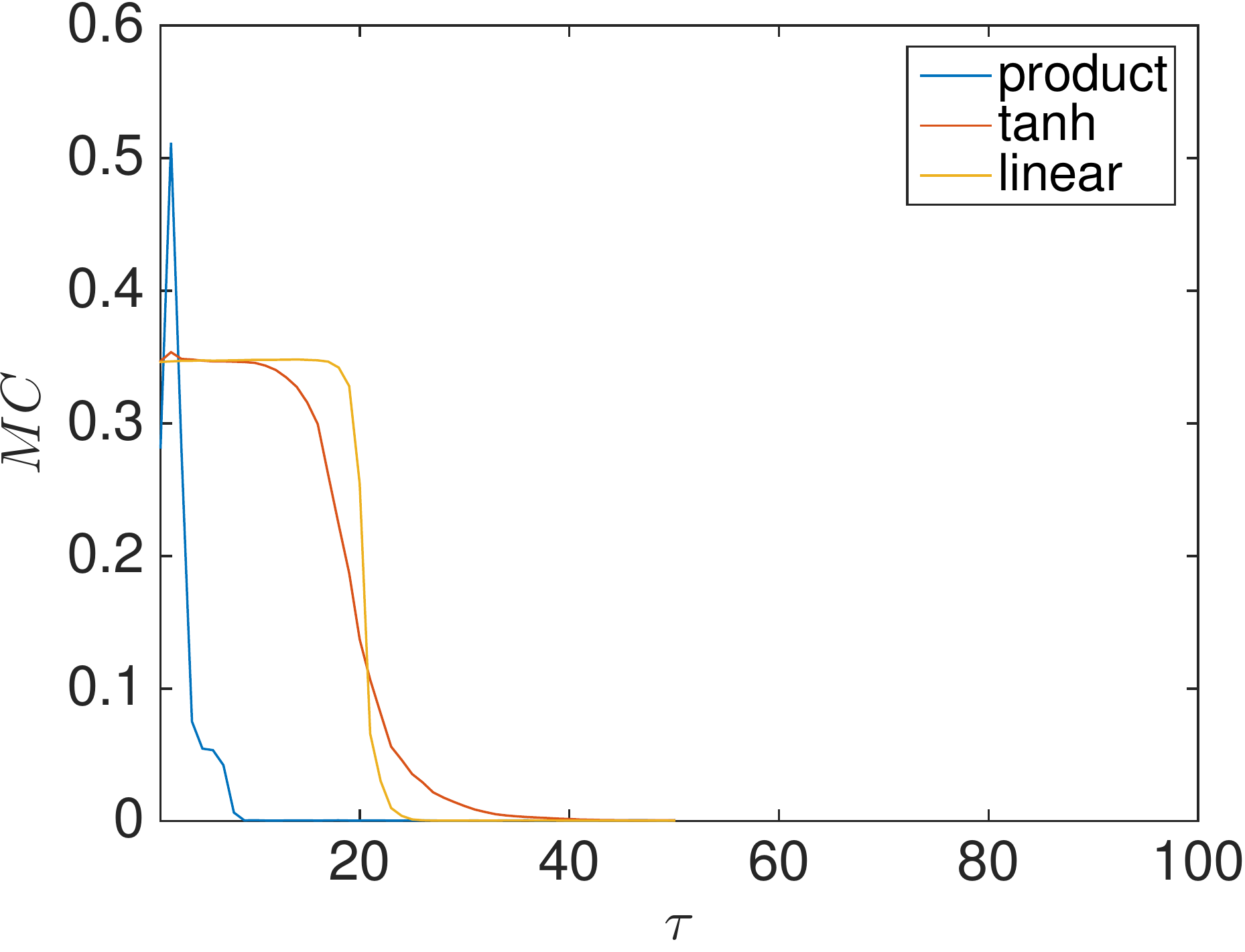}
\label{fig:bestNMC3}
}
\caption{(a) Nonlinear computation capacity of ESN with product nodes, {\it{tanh-}} nodes, and linear nodes. With the exception of $n=3$, the product RC outperform the {\it{tanh-}} ESN in nonlinear capacity. (b) shows the complete nonlinear capacity function for $n=3$. The product network exhibits  long-term memory whereas the {\it{tanh-}} ESN exhibits short-term memory.}
\label{fig:nmc}
\end{figure}

Figure~\ref{fig:bestNMC} summarizes the nonlinear computation capacity for $2\le n\le8$. Product RC clearly shows useful nonlinear computation. However, a complete and fair comparison between the nonlinear capacity of product RC and of $\tanh$ ESN is beyond the scope of this work. First, t has already been reported that the {\it{tanh-}} ESN are unable to perform nonlinear memory tasks with even degrees \cite{Dambre:2012fk}, but adding bias to the reservoir in these networks fixes this problem. In addition, our preliminary experiment shows that multiplicative readout layer in product RCs also significantly improves their performance.
\begin{figure}[t]
\centering
\def\w{1.6in}
\subfloat[product, linear memory]{
\includegraphics[width=\w]{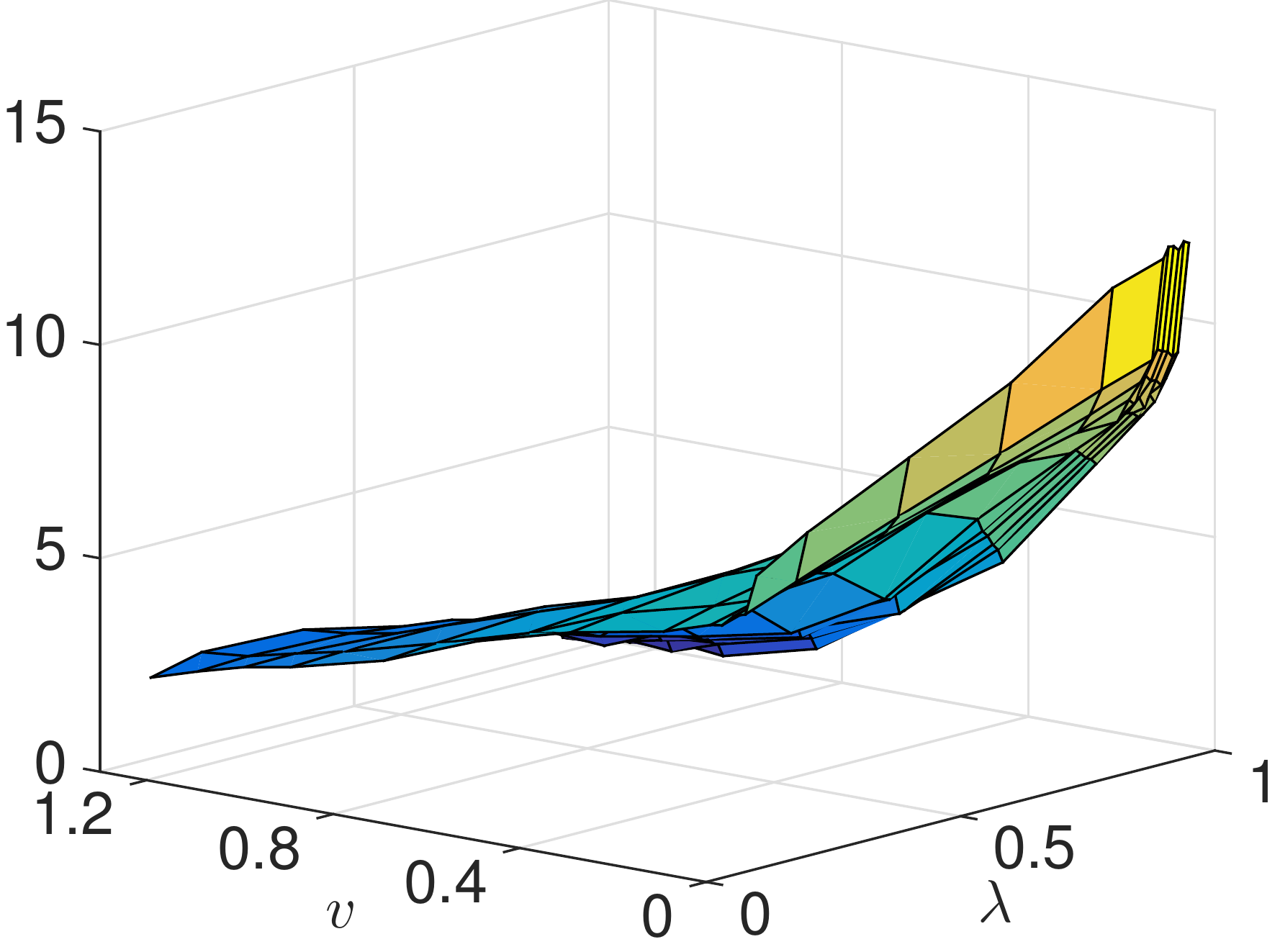}
}
\subfloat[{\it{tanh}}, linear memory]{
\includegraphics[width=\w]{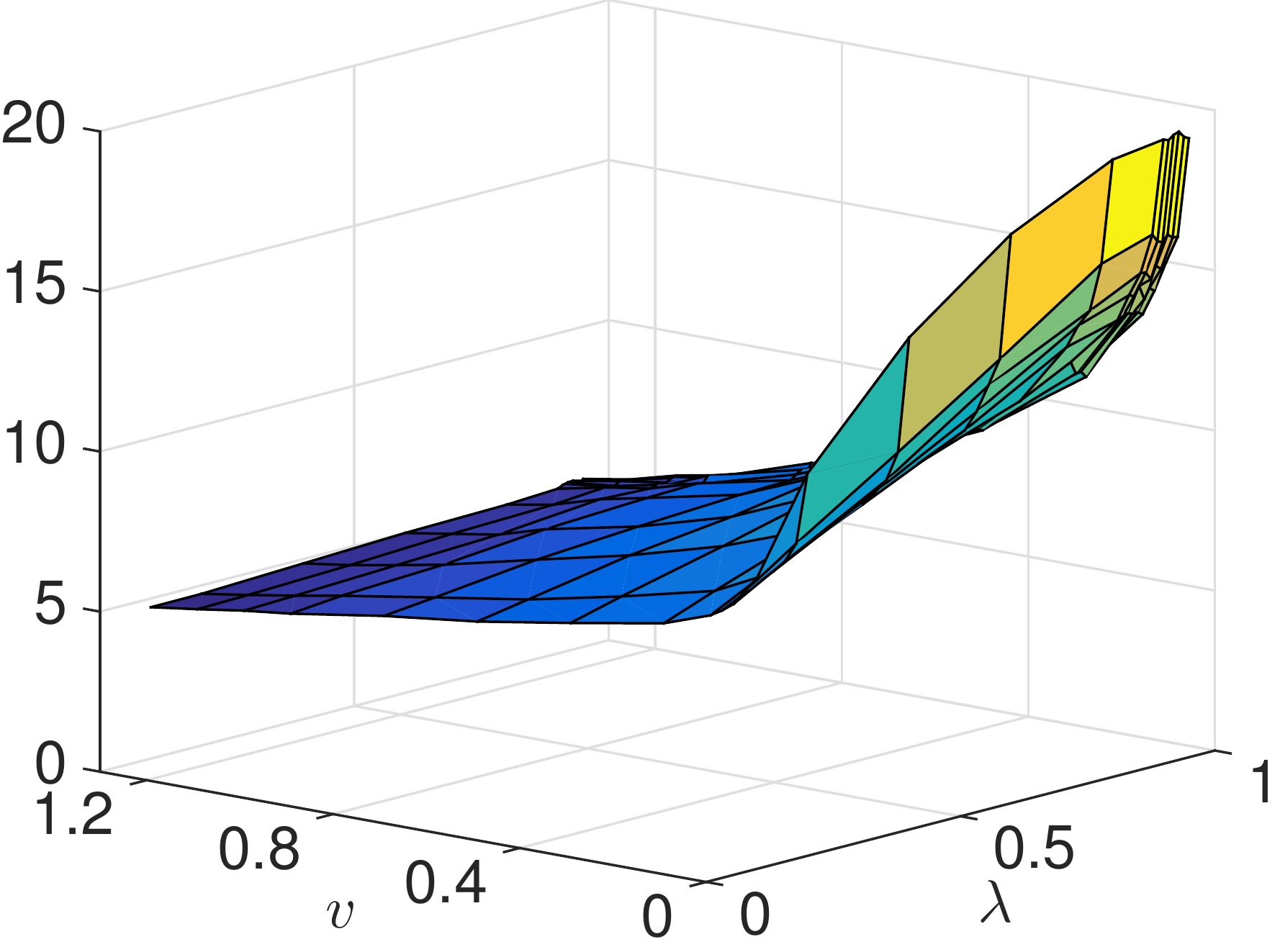}
}
\\
\subfloat[product, nonlinear capacity, $n=3$]{
\includegraphics[width=\w]{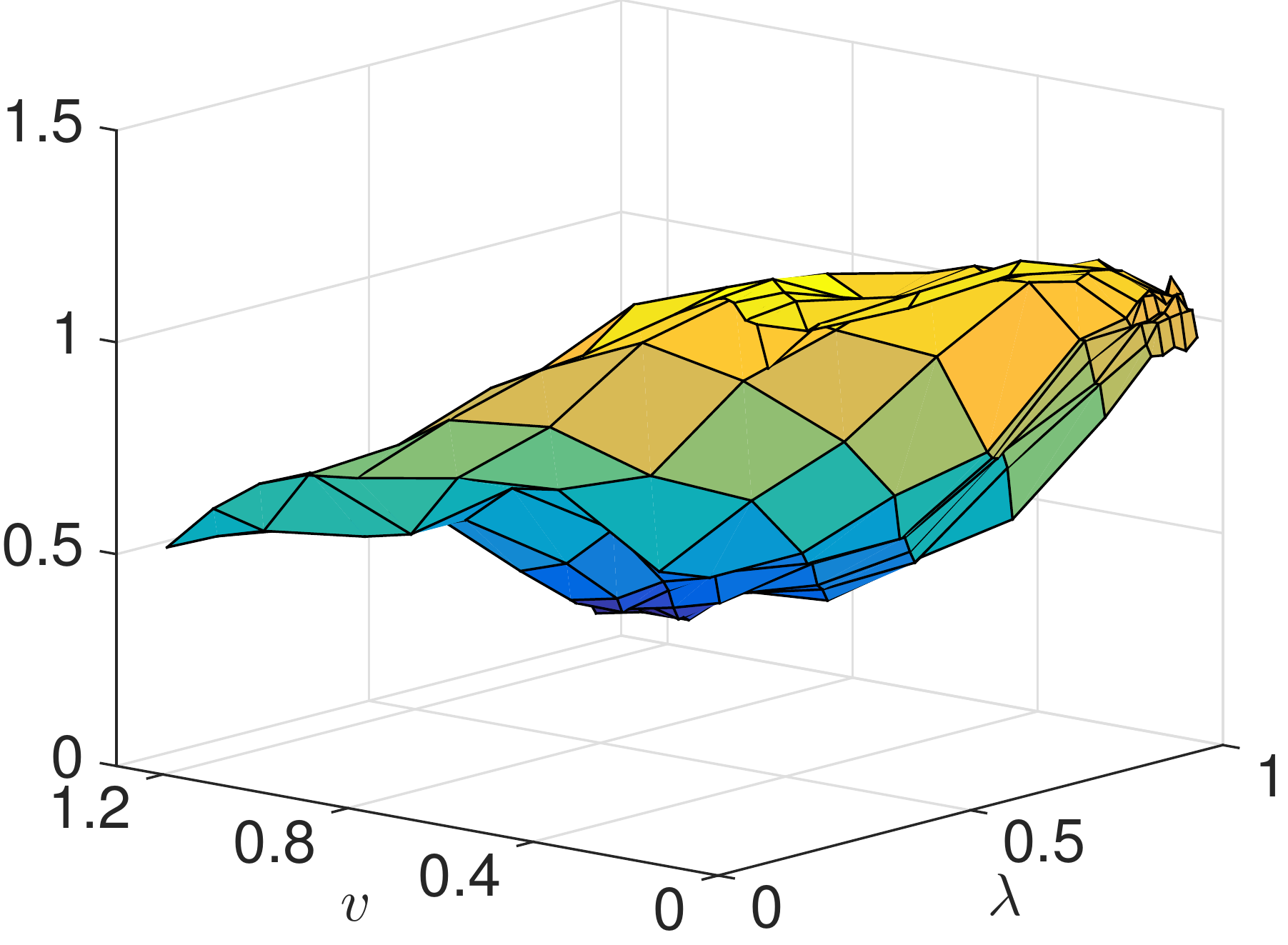}
}
\subfloat[{\it{tanh}}, nonlinear capacity, $n=3$]{
\includegraphics[width=\w]{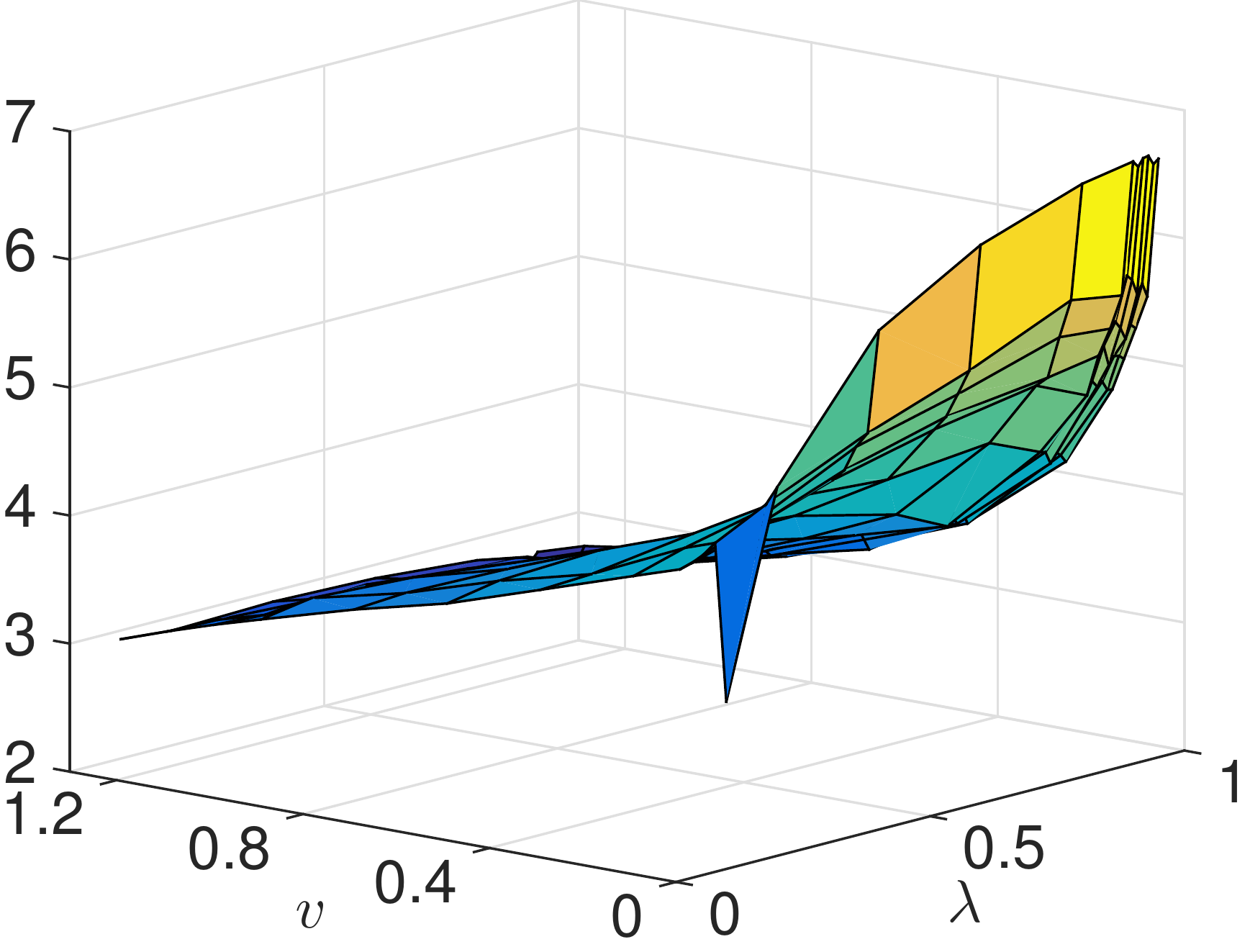}
}
\\
\subfloat[product, nonlinear capacity, $n=5$]{
\includegraphics[width=\w]{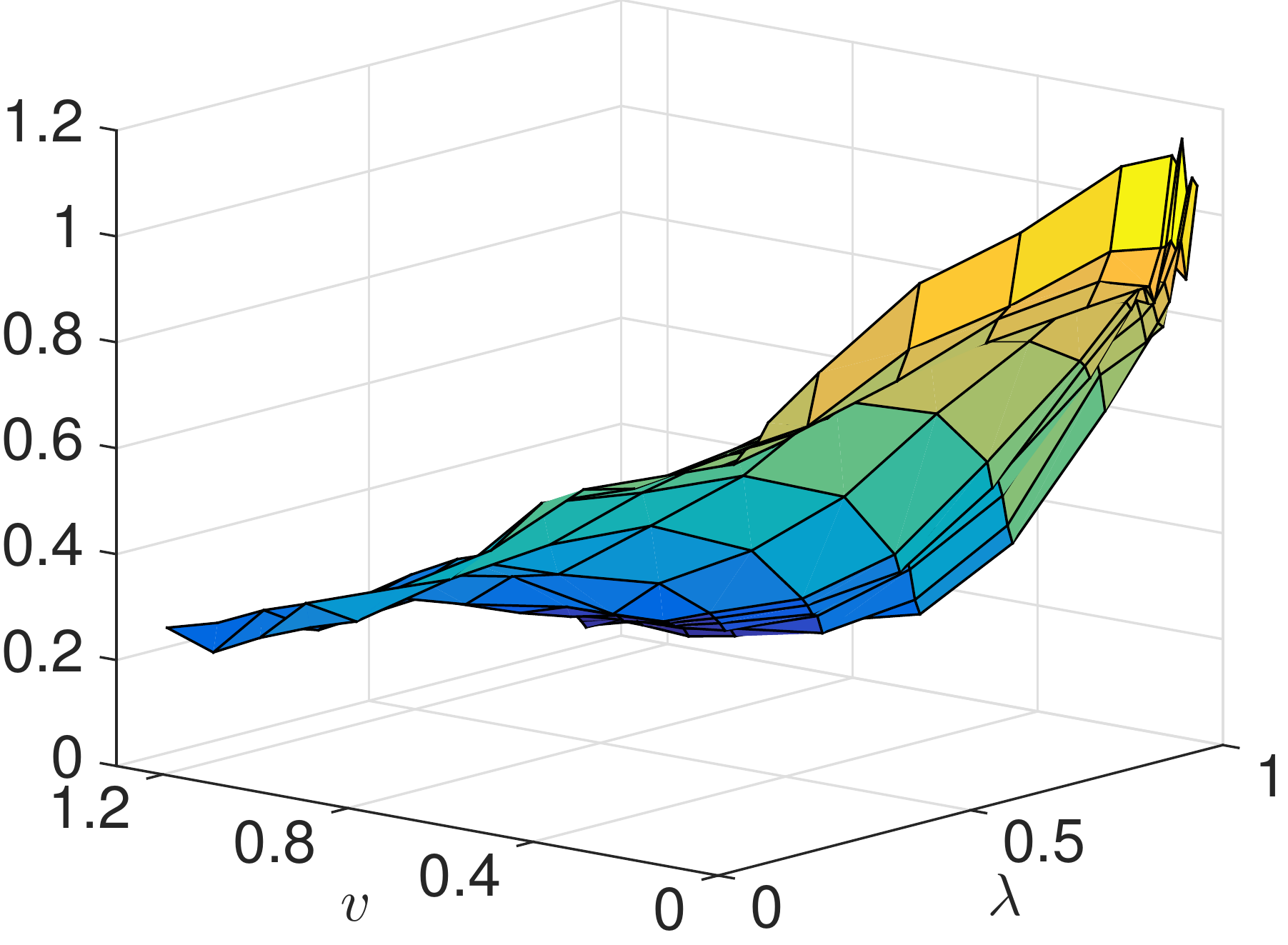}
}
\subfloat[{\it{tanh}}, nonlinear capacity, $n=5$]{
\includegraphics[width=\w]{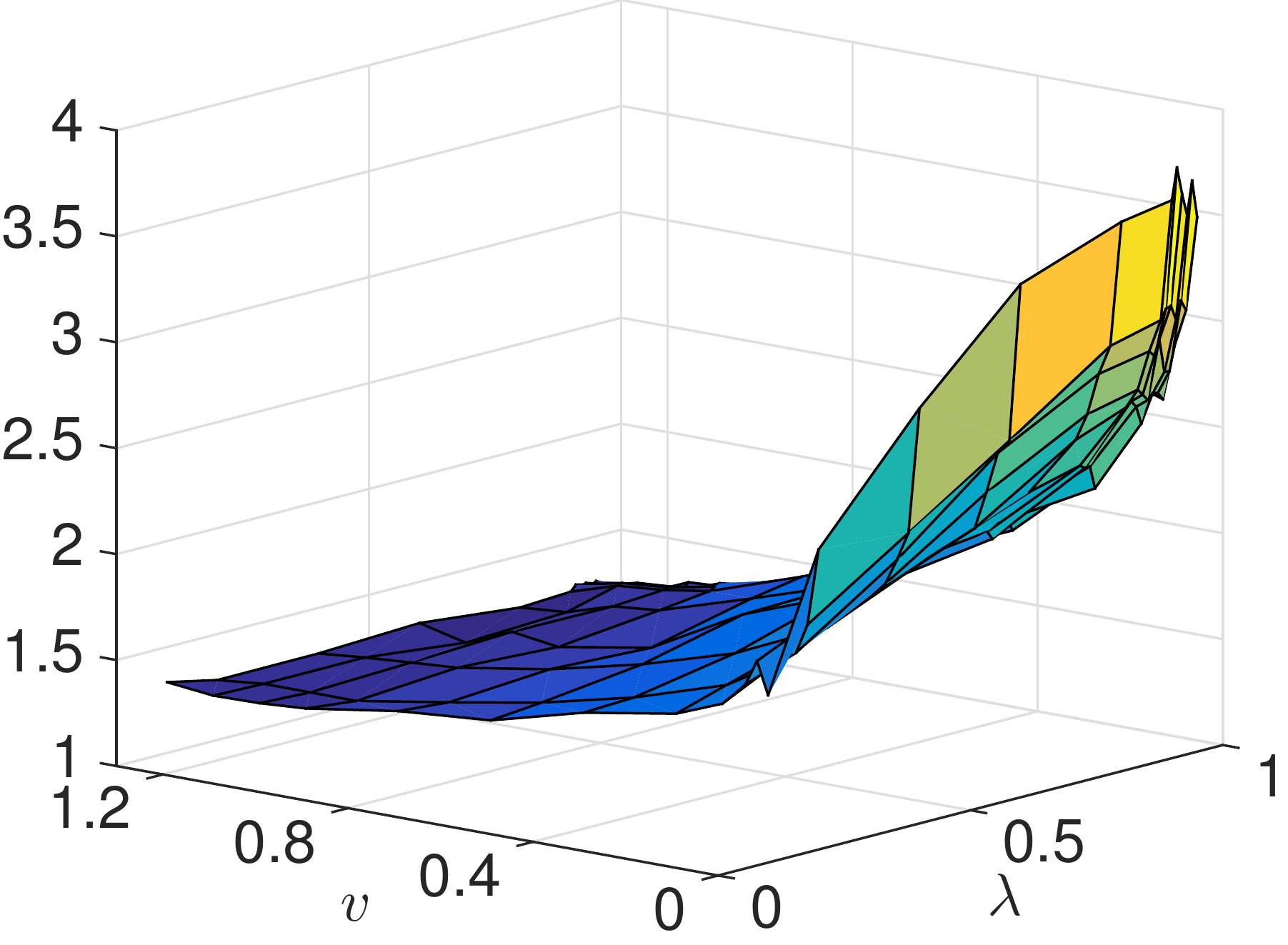}
}
\caption{Sensitivity analysis of the linear memory capacity and the nonlinear computation capacities of order $n=3,5$. The product RC exhibit best memory capacity for high $\lambda$ and low $\omega$, whereas for the {\it{tanh-}} ESN the optimal parameters depend on the type of memory measured.}
\label{fig:NMCsurf}
\end{figure}
To understand the difference between the nonlinear capacity of the product network and the {\it{tanh-}} ESN for $n=3$, we must look at their capacity functions over $\tau$. Figure~\ref{fig:bestNMC3} shows the 3rd order capacity function for the three types of networks as a function of time. The capacity function of each network is chosen for the optimal parameter set of that network. The {\it{tanh-}} ESN can perfectly reconstruct the desired output for just a few recent inputs (short-term memory), while the product RC cannot reconstruct correct values perfectly, but it can do it with larger $\tau$, i.e., longer input histories. This behavior is analogous to high-quality short-term memory in recurrent networks operating in linear regime versus low-quality long-term memory  in nonlinear regime \cite{5596492}. The memory and nonlinear capacity results in this work do not 
consider the statistical significance test  and
only show  qualitative features of product RCs and ESNs. For accurate estimation of the exact
values one need to perform the measurement as described in \cite{Dambre:2012fk}. Also for simplicity we have not applied any reservoir bias to $\tanh$ ESNs, which is known to 
improve the nonlinear capacity of $\tanh$ reservoirs.

Figure~\ref{fig:NMCsurf} shows the sensitivity analysis of memory and nonlinear capacity of both
product RC and $\tanh$ ESN to input weight scaling and reservoir spectral radius. As expected, both product and $\tanh$ reservoirs perform best with high spectral radius and low input weight scaling. Next, we see how the product RC and the standard ESN compare in solving signal processing benchmarks.




\subsection{Chaotic Time Series Prediction}
\subsubsection{Mackey-Glass System Prediction}

The Mackey-Glass system \cite{Mackey15071977} is a delayed differential equation with chaotic dynamics, commonly used as a benchmark for chaotic signal prediction. This system is described by:
\begin{align}
\frac{dx_t}{dt} =  \beta \frac{x_{t-\delta}}{1+x_{t-\delta}^n} - \gamma x_t,
\end{align}
where $\beta=0.2,n=10$, and $\gamma=0.1$ are positive constants and $\delta=17$ is the feedback delay. The reservoir consists of $N=500$ nodes, and we systematically vary the input weight coefficients and the spectral radius in the range $0.1<\omega<1$ and $0.1<\lambda<0.9$. The task is to predict the next $\tau$ integration time steps given $x_t$. We scaled the time series between $[0,1]$ before feeding the network.

\begin{figure*}[ht!]
\centering
\def\w{2in}
\subfloat[product, Mackey-Glass]{
\includegraphics[width=\w]{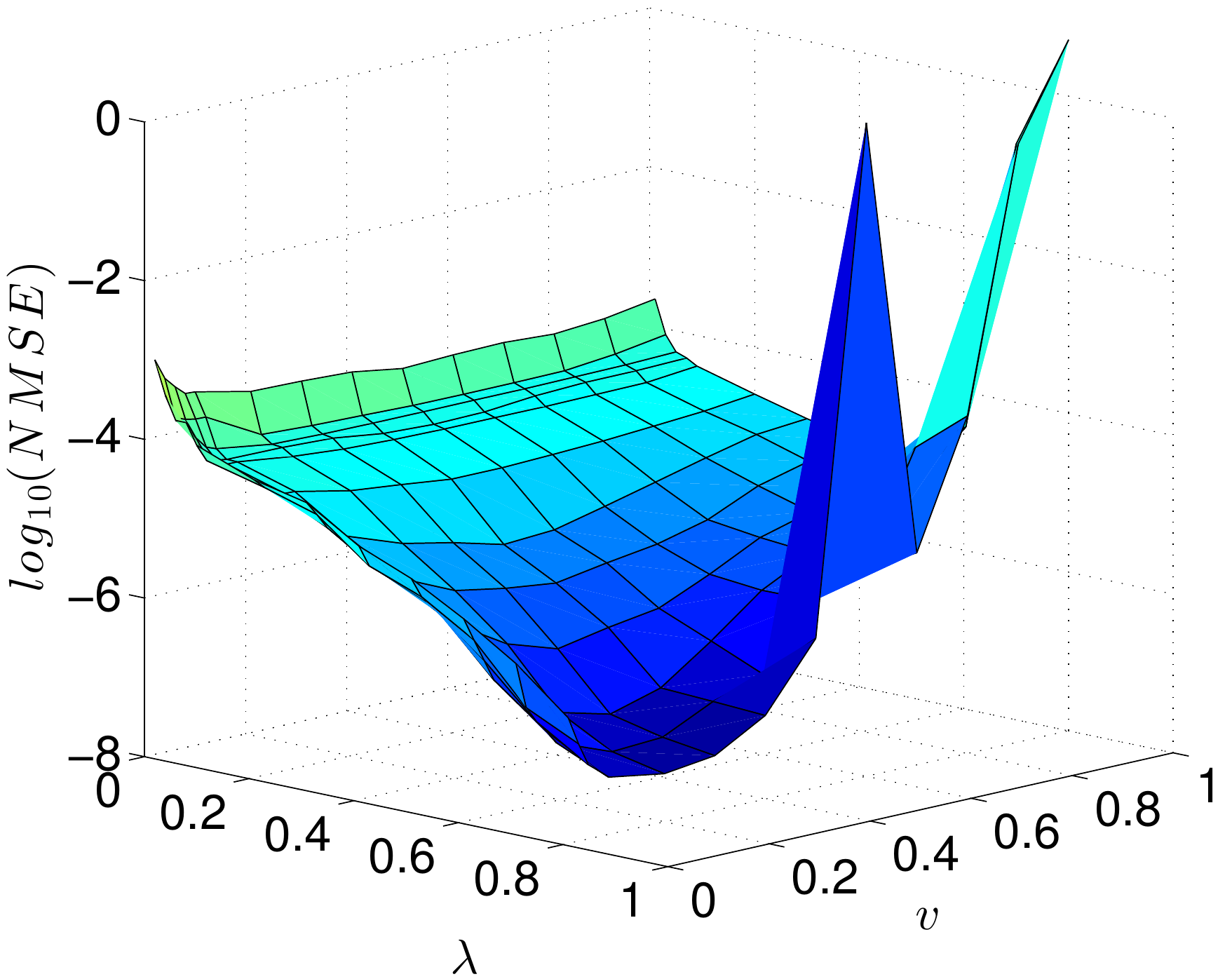}
\label{fig:prodlinMG}
}
\subfloat[{\it{tanh}}, Mackey-Glass]{
\includegraphics[width=\w]{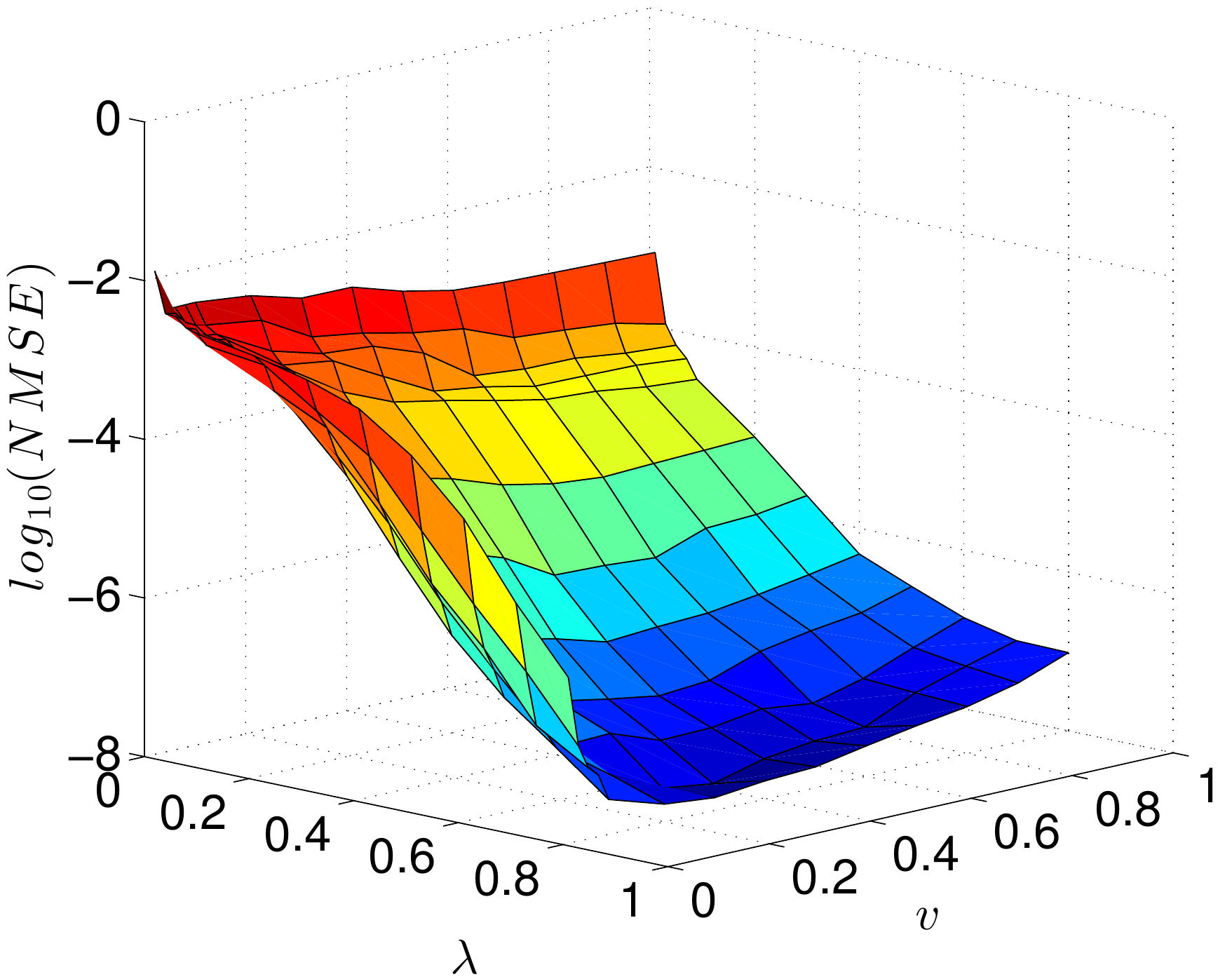}
\label{fig:tanhMG}
}
\subfloat[linear, Mackey-Glass]{
\includegraphics[width=\w]{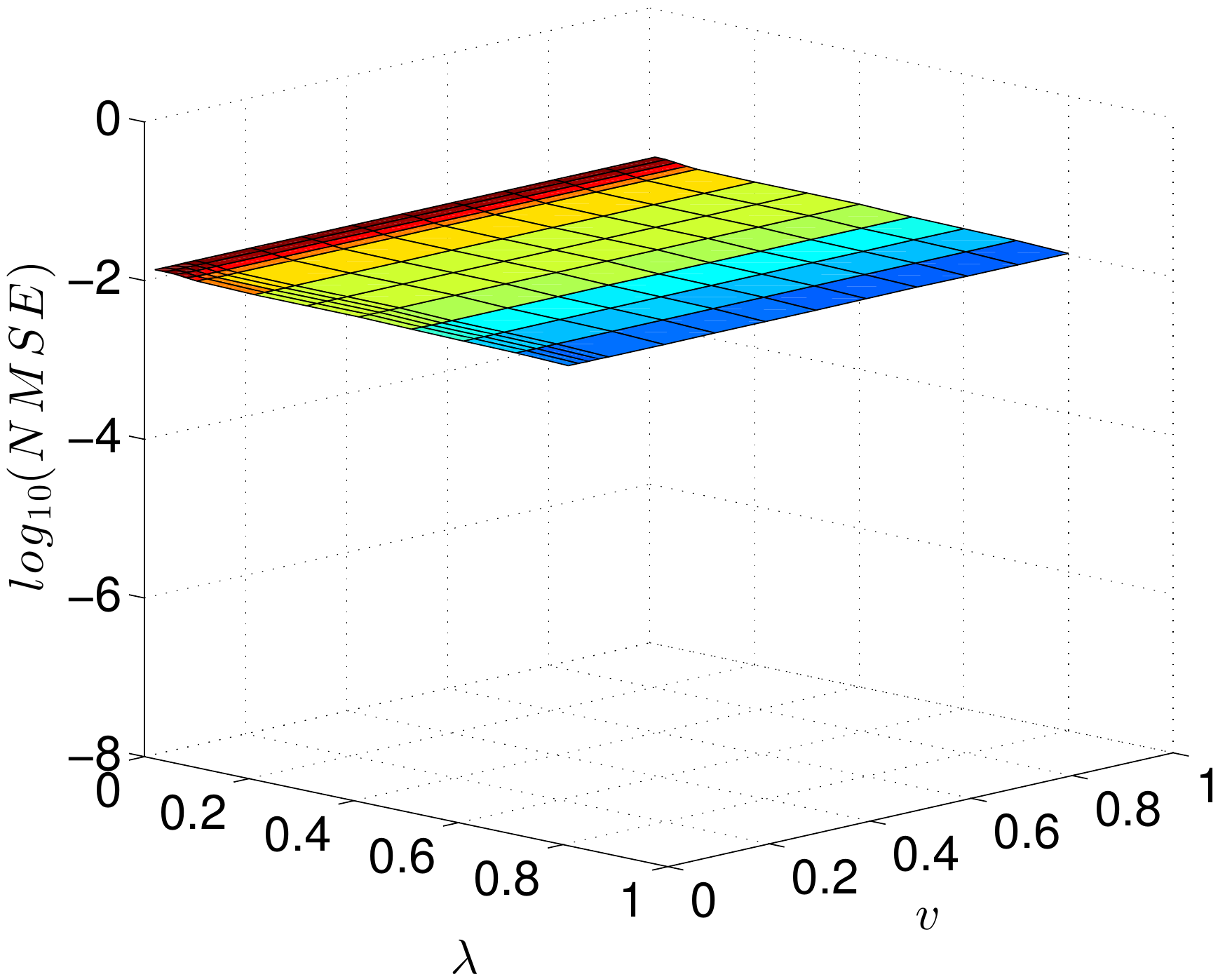}
\label{fig:linMG}
}\\
\subfloat[product, Lorenz]{
\includegraphics[width=\w]{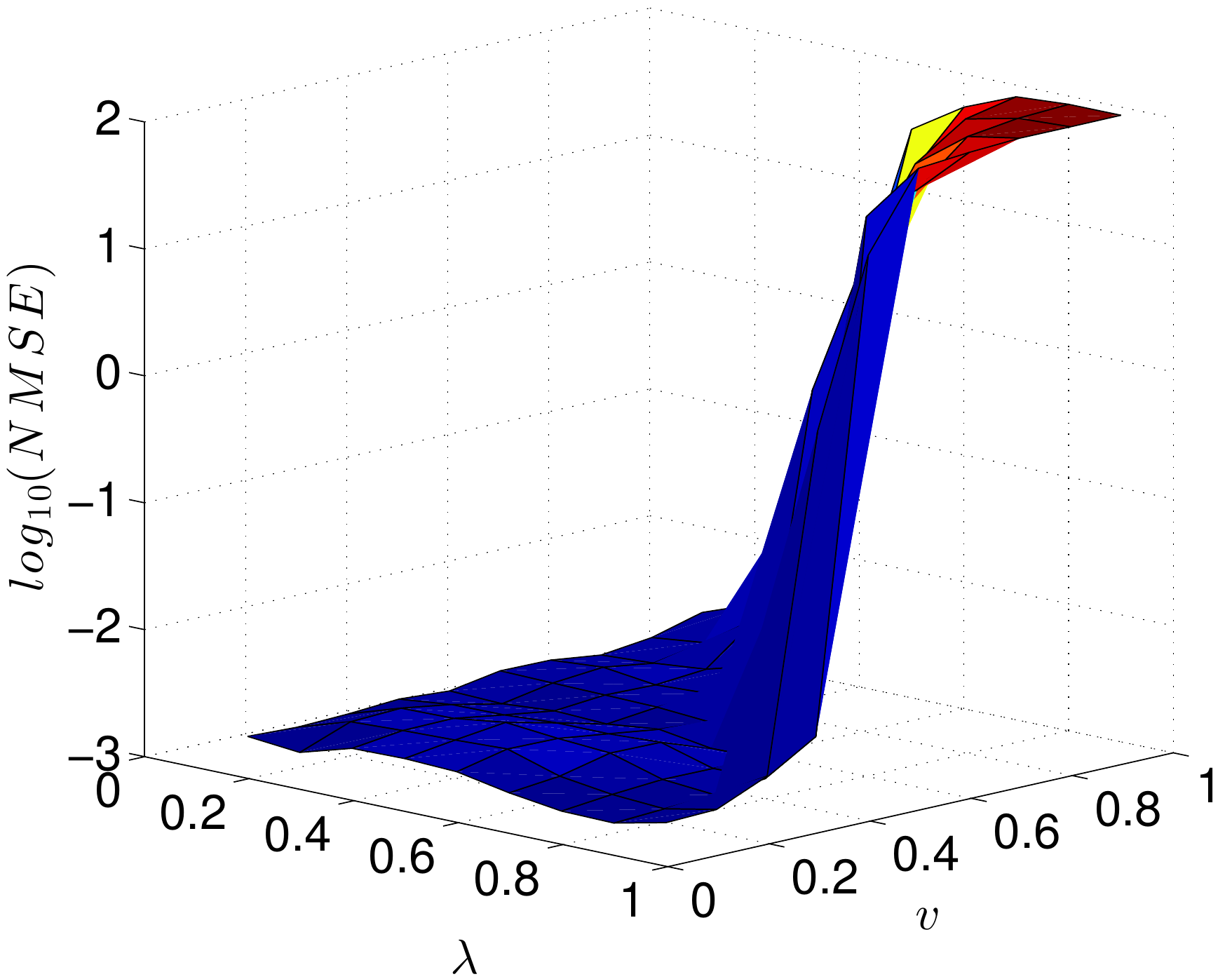}
\label{fig:prodlinLZ}
}
\subfloat[{\it{tanh}}, Lorenz]{
\includegraphics[width=\w]{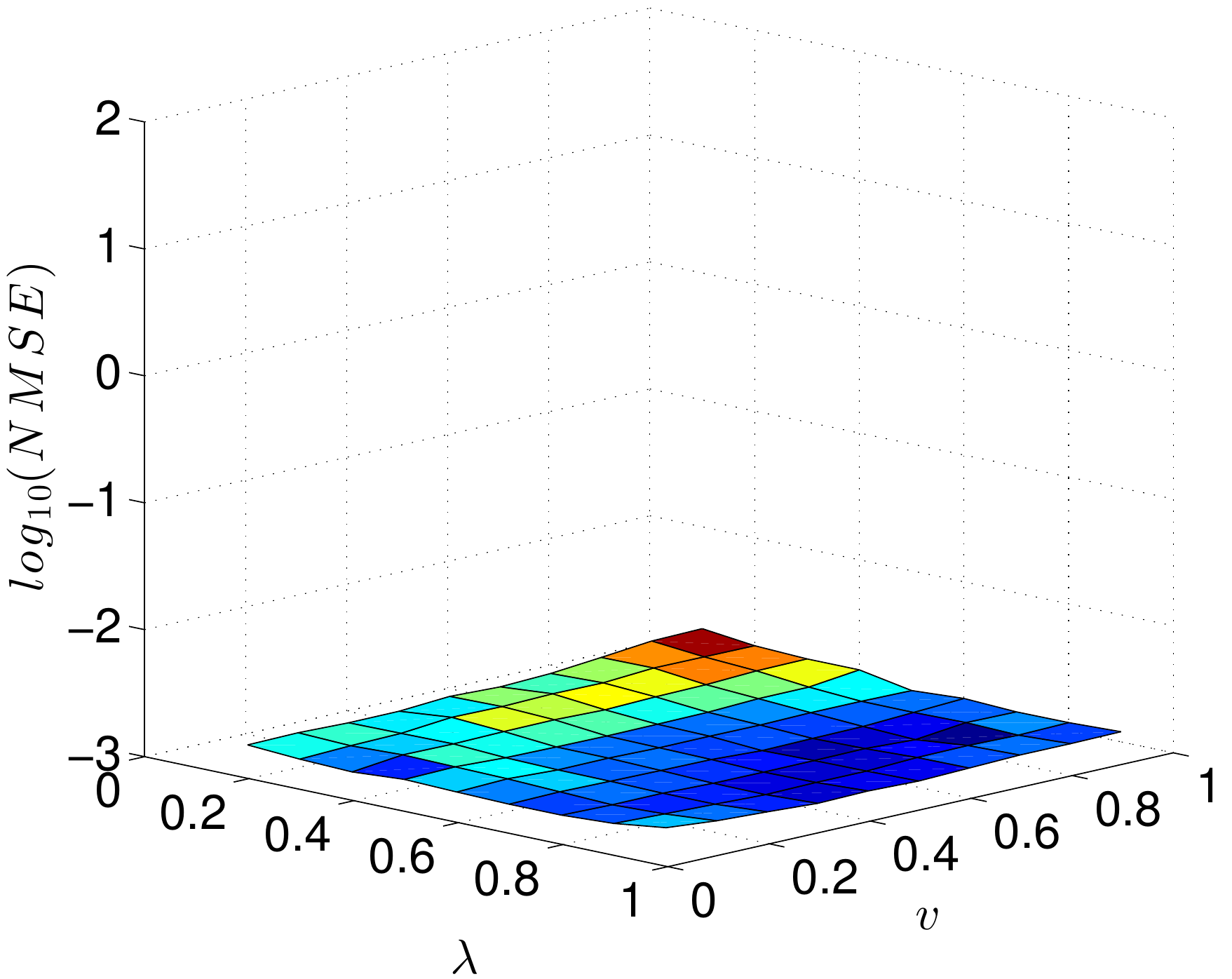}
\label{fig:tanhLZ}
}
\subfloat[linear, Lorenz]{
\includegraphics[width=\w]{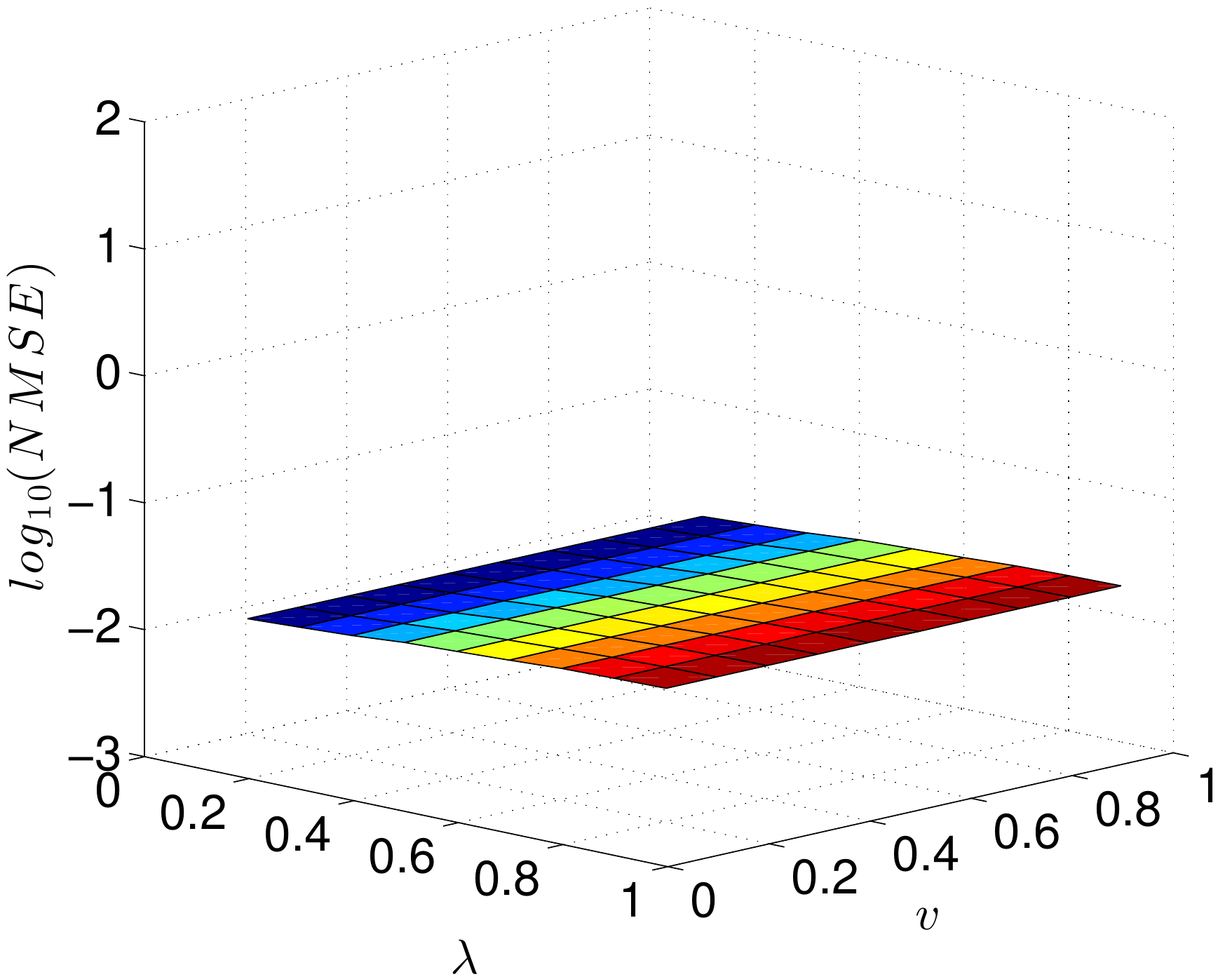}
\label{fig:linLZ}
}
\caption{The performance of one step prediction of the Mackey-Glass and the Lorenz system. The best performance for product RCs are almost identical to the standard ESN with {\it{tanh}} activation function. The standard linear ESN is included for comparison.}
\label{fig:MGLZfigures}
\end{figure*}

\subsubsection{Lorenz System Prediction}
The Lorenz system is another standard benchmark task for chaotic prediction.
The Lorenz system \cite{Lorenz1963} is defined as:

\begin{align}
\begin{split}
\frac{dx_t}{dt} &= \sigma ( y_t-x_t), \\
\frac{dy_t}{dt} &= x_t(\rho - z ) -y_t, \\
\frac{dz_t}{dt} &= x_ty_t - \beta z_t,
\end{split}
\end{align}
where $\beta = 8/3$, $\rho = 28$, and $\sigma=10$. These values give rise  to chaotic dynamics, making the system a suitable benchmark for multi-dimensional chaotic time-series prediction. The reservoir consists of $N=500$ nodes, and we systematically vary the input weight coefficients and the spectral radius in the range $0.1<\omega<1$ and $0.1<\lambda<0.9$. We feed all three variables to our systems, after scaling each variable on the interval $[0,1]$. The task is to produce the next $\tau$ integration time steps for all three variables. We evaluate the performance $NMSE_{tot}$ by calculating $NMSE$ for each output and adding them together.


\subsubsection{Results}

Figure~\ref{fig:MGLZfigures} shows the performance of the product and standard ESN in predicting the next time step of the time series. Product RC achieves comparable accuracy level as standard ESN with {\it{tanh}} activation (see Figures~\ref{fig:prodlinMG}~and~\ref{fig:tanhMG}). Similarly, for the Lorenz system both product and {\it{tanh-}} ESNs show a similar performance at their optimal parameters (see Figures~\ref{fig:prodlinLZ}~and~\ref{fig:tanhLZ}). We have included the linear ESN for comparison. In our experiments, the parameters $\omega=0.1$ and $\lambda=0.8$ are optimal for the product and the {\it{tanh-}} ESNs for both tasks. The full sensitivity analysis reveals that the product and the {\it{tanh-}} ESNs show task-dependent sensitivity in different parameter ranges. For example, for the product RC on the Mackey-Glass task, decreasing $\lambda$ increases the error by 4 orders of magnitude, whereas the {\it{tanh-}} ESN's error increases by 5 orders of magnitude. On the other hand, the {\it{tanh-}} ESN is robust to increasing $\lambda$, while the product RC loses its power due to instability in the network dynamics. For the Lorenz system, high $\omega$ and $\lambda$ destabilizes the product RC dynamics, which completely destroys the power of the system. However, the performance of the {\it{tanh-}} ESN does not vary significantly with changes in the parameter, because the Lorenz system does not require any memory. The linear ESN does not show any sensitivity to the parameter space.  

We then use the optimal values of the parameters to test and compare the performance of multi-step prediction of both Mackey-Glass and Lorenz systems. The task is for the system to produce the correct values 
\begin{figure}
\centering
\subfloat[Mackey-Glass]{
\includegraphics[width=1.6in]{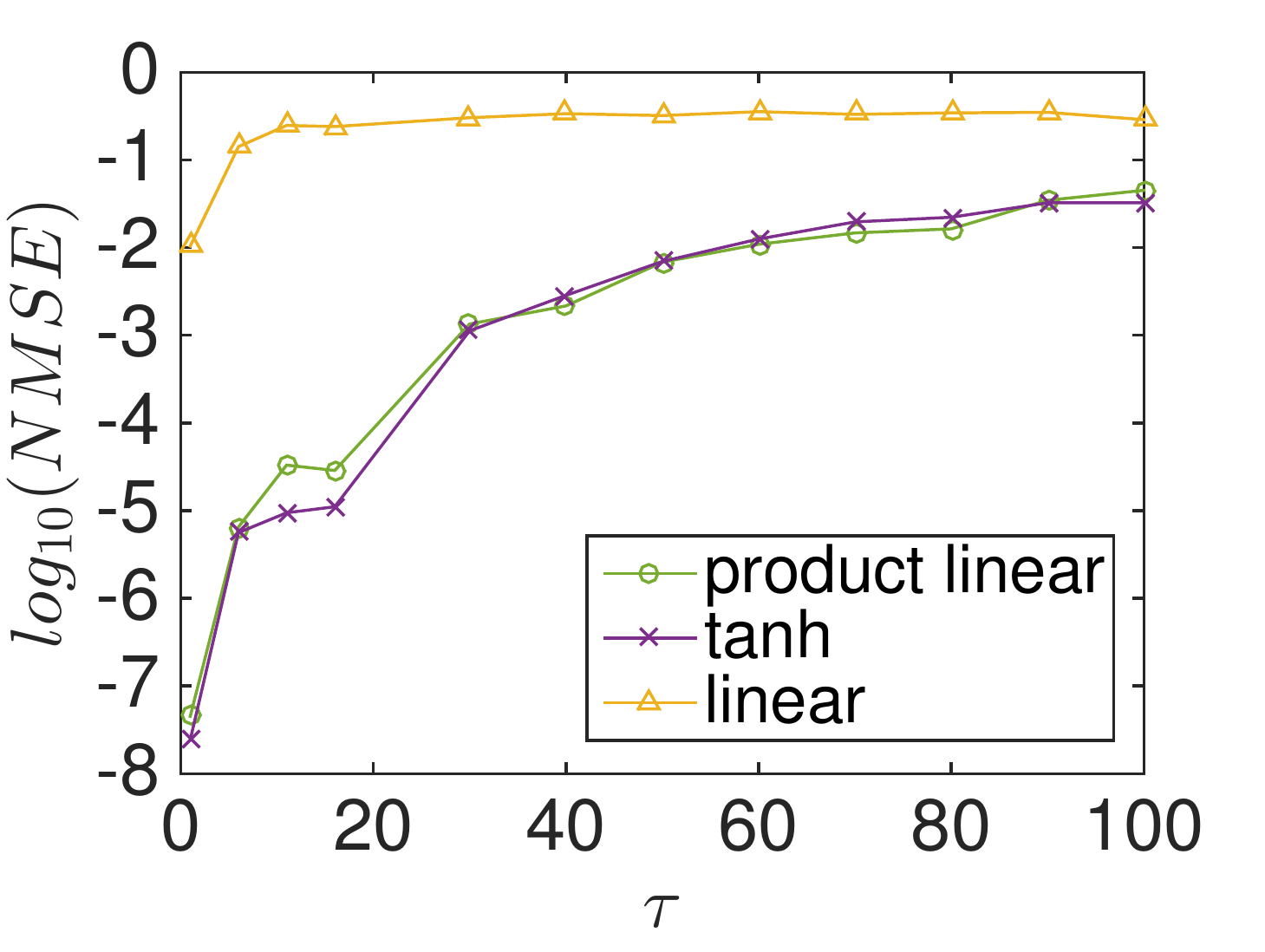}}
\subfloat[Lorenz]{
\includegraphics[width=1.6in]{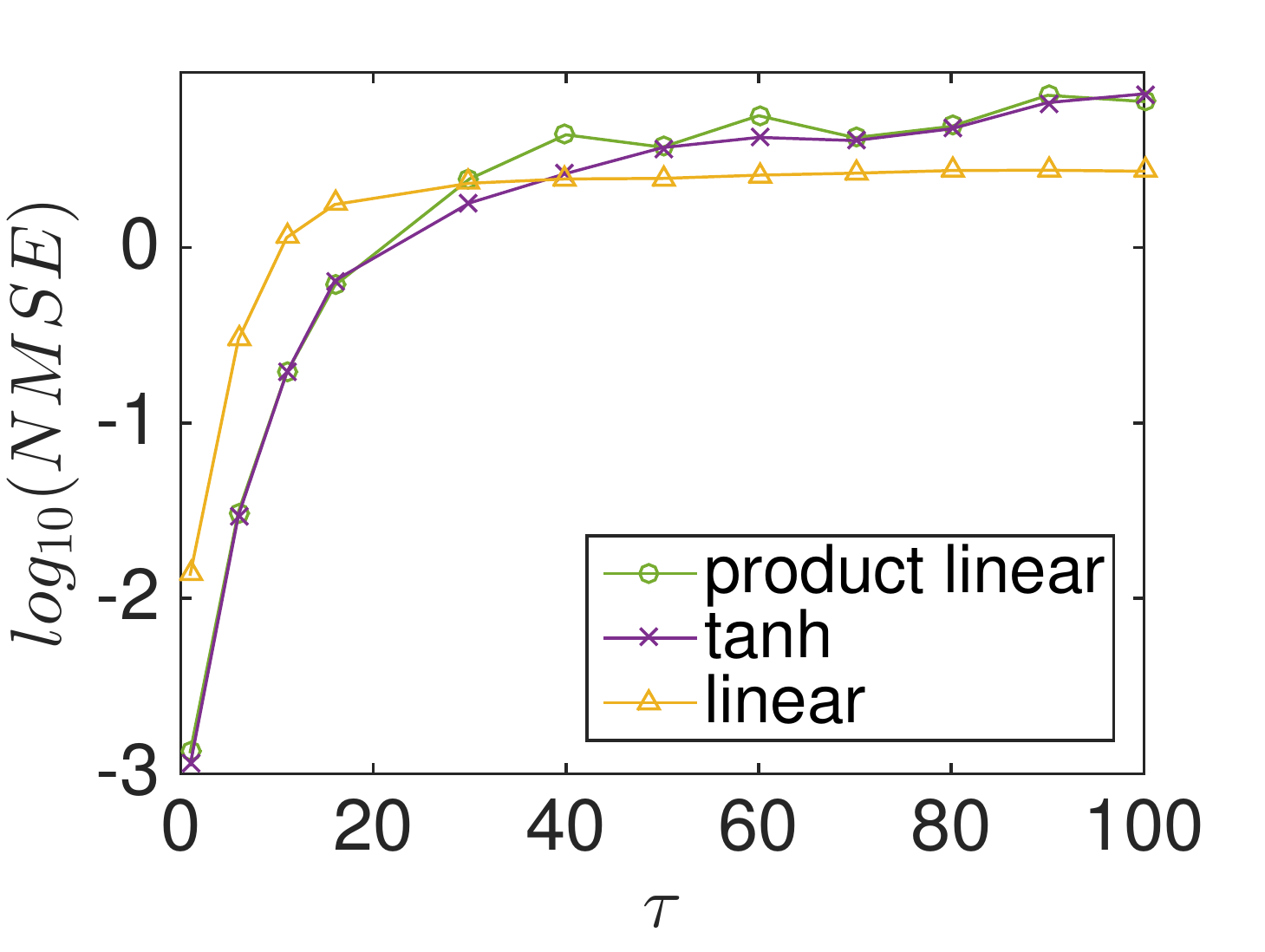}}
\caption{The performance of multi-step prediction of the Mackey-Glass and the Lorenz systems. product RCs perform almost identically to the standard ESN with {\it{tanh}} activation function. The standard linear ESN is included for comparison.}
\label{fig:MGmultistep}
\end{figure}
for the systems $\tau$ step ahead. Figure~\ref{fig:MGmultistep} shows the $\text{log}_{10}(NMSE)$ for different $\tau$. The product RCs show performance quality similar to the standard {\it{tanh-}} ESNs. The standard ESN with linear activation is included for comparison.

\section{Conclusion and outlook}

Nonlinearity of neural response is essential for real-time computational tasks that involve strong time variations of the input data. In this manuscript, we considered neural networks with a basic nonlinear property that  a neuron  outputs a weighted product of synaptic inputs. The presented modeling is an abstract representation of some type of neurons, whose response function is not fully captured by commonly used synaptic sums followed by a {\it{tanh}} thresholding function. We evaluated the performance of a neural network with such product units in the computational paradigm of reservoir computing. Product RCs were found to be comparably powerful for nonlinear computation. For the nonlinear capacity task we have used the standard versions of product RC
and $\tanh$ ESNs for simplicity. In our preliminary experiments we have observed that 
the use of bias in $\tanh$ ESNs and multiplicative readout layer for product RCs can
significantly improve their performance. We defer a fuller analysis of these architectural 
variations to future work. On standard tests, we found that for Mackey-Glass and Lorenz systems prediction, a network of product units performs  as well as a network of {\it{tanh}} units. For both types of nonlinear networks we found that the best performance is achieved with  relatively small input weights, which does not take advantage of the full nonlinearity of the reservoir nodes. For {\it{tanh-}} networks, this nonlinear advantage will completely disappear for very small input weights. We will present a detailed study of this subtle behavior in a forthcoming paper \cite{goudarzi2015b}. Overall, our findings suggest that neural networks with product neurons  may have stronger capacity than {\it{tanh}} neurons for certain real-time data processing tasks.


\section*{Acknowledgment}
This material is based upon work supported by the National Science Foundation under grants CDI-1028238 and CCF-1318833.


\bibliographystyle{IEEEtran}
\bibliography{ijcnn2015}

\end{document}